\documentclass{article}

\usepackage[final,main]{nips/neurips_2026}

\usepackage{amsmath}
\usepackage{amssymb}
\usepackage{amsfonts}
\usepackage{amsthm}
\usepackage{mathrsfs}
\usepackage{thmtools}
\usepackage{thm-restate}

\usepackage{graphicx}
\usepackage{booktabs}
\usepackage{textgreek}
\usepackage{xcolor}
\usepackage[most]{tcolorbox}
\usepackage{todonotes}
\usepackage{url}
\usepackage{nicefrac}
\usepackage{microtype}
\usepackage[T1]{fontenc}
\usepackage[utf8]{inputenc}

\theoremstyle{definition}
\newtheorem{definition}{Definition}[section]

\newtheorem{example}[definition]{Example}

\theoremstyle{plain}
\newtheorem{theorem}[definition]{Theorem}
\newtheorem{proposition}[definition]{Proposition}
\newtheorem{lemma}[definition]{Lemma}

\definecolor{messagecolor}{RGB}{245,245,245}
\definecolor{bordercolor}{RGB}{180,180,180}

\newtcolorbox{conclusionbox}[1]{
  enhanced,
  arc=4pt,
  boxrule=1pt,
  colframe=cyan!70!black,
  colback=cyan!4!white,
  coltitle=black,
  fonttitle=\bfseries,
  attach boxed title to top left={xshift=15pt,yshift=-6pt},
  boxed title style={
    colback=white,
    colframe=cyan!70!black,
    arc=3pt,
    boxrule=1pt
  },
  title=#1
}

\newtcolorbox{messagebox}[1]{
  colback=messagecolor,
  colframe=bordercolor,
  fonttitle=\bfseries,
  coltitle=black,
  title=#1,
  enhanced,
  attach title to upper,
  after title={:\enskip},
  sharp corners,
  boxrule=0.5pt,
  left=5pt,
  right=5pt,
  top=5pt,
  bottom=5pt
}

\newcommand{\G}{\mathbb{G}}

\newcommand{\N}{\mathcal{N}}
\newcommand{\bbN}{\mathbb{N}}

\newcommand{\R}{\mathbb{R}}
\newcommand{\RR}{\mathbb{R}}

\newcommand{\csort}{c_{\mathrm{lex}}}
\newcommand{\Lemp}{\mathcal{L}_{\mathrm{emp}}}
\newcommand{\Lexp}{\mathcal{L}_{\mathrm{exp}}}

\newcommand{\twopartdef}[4]{
  \left\{
  \begin{array}{ll}
    #1 & \mbox{if } #2 \\
    #3 & \mbox{if } #4
  \end{array}
  \right.
}

\newcommand{\threepartdef}[6]{
  \left\{
  \begin{array}{lll}
    #1 & \mbox{if } #2 \\
    #3 & \mbox{if } #4 \\
    #5 & \mbox{if } #6
  \end{array}
  \right.
}

\usepackage{xcolor}

\usepackage{hyperref}

\title{When and How to Canonize: a Generalization Perspective}

\author{%
  Yonatan Sverdlov\thanks{Corresponding author.}\\
  Technion -- Israel Institute of Technology\\
  \texttt{yonatans@campus.technion.ac.il}
  \And
  Benjamin Friedman\\
  Technion -- Israel Institute of Technology\\
  \texttt{benjamin.fri@campus.technion.ac.il}
  \AND
  Snir Hordan\\
  Technion -- Israel Institute of Technology\\
  \texttt{snirhordan@campus.technion.ac.il}
  \And
  Nadav Dym\\
  Technion -- Israel Institute of Technology\\
  \texttt{nadavdym@technion.ac.il}
}

\newcommand{\ConferenceBibliography}{%
  \bibliographystyle{plain}%
  \bibliography{can}%
}

\begin{document}

\maketitle

\begin{abstract}

While invariant architectures are standard for processing symmetric data, there is growing interest in achieving invariance by applying group averaging or canonization to non-invariant backbones. However, the theoretical generalization properties of these alternative strategies remain poorly understood. We introduce a theoretical framework to analyze the generalization error of these methods by bounding their covering numbers. We establish a rigorous generalization hierarchy: the error bounds of canonized models are at best equal to the error bounds of structurally invariant and group-averaged models, and at worst equal to the bounds of non-invariant baselines. Furthermore, we show that there exist "optimal" canonizations which attain the optimal error bounds, and "poor" canonizations which attain the non-invariant error bounds, and that this depends on the regularity of the canonization.
Finally, applying this framework to permutation groups in point cloud processing, we rigorously prove that the covering number of lexicographical sorting grows exponentially with the number of points in the point cloud, whereas Hilbert curve canonization guarantees polynomial growth.
This provides the first formal theoretical justification for the empirical success of Hilbert curve serialization in state-of-the-art point cloud architectures. We conclude with experiments that support our theoretical claims. 
Our code is publicly available on \href{https://github.com/yonatansverdlov/Canonization}{GitHub}.

\end{abstract}
\section{Introduction}
The integration of geometric priors into neural networks has become a cornerstone of modern machine learning, particularly for applications involving graphs, 3D point clouds, and molecular structures. In these domains, the underlying data often exhibits inherent symmetries, meaning the target function to be learned is invariant to a specific group action. Exploiting these symmetries is known to improve sample complexity and generalization \cite{elesedy2021provably,brehmer2025doesequivariancematterscale}.

While specialized architectures, which are invariant by design, are a common method for processing symmetric data, there is growing interest in achieving invariance by applying generalized group averaging methods to non-invariant backbones.

This family of methods includes full group averaging, which is theoretically sound but intractable for all but very small groups, group augmentation, which encourages invariance by randomly sampling group transformations during training, frame averaging \cite{punyframe}, which enables invariant averaging over subsets of the group, and canonization, which maps every group orbit to a single, consistent "canonical" element.

Canonization is the most efficient of all generalized group averaging methods, as it only involves processing a single orbit representative. At the same time, it enjoys the same universal approximation guarantees as full group averaging \cite{kaba2023equivariance}. These observations have motivated researchers to define canonizations in several diverse domains. These include sign canonizations for spectral embeddings of graphs \cite{ma2023laplacian,hordan2025spectral}, canonizations of graphs \cite{lin2024equivariance} and point clouds \cite{kaba2023equivariance,pmlr-v235-baker24a,friedmann2025,zhou2026rethinkingdiffusionmodelssymmetries} including the celebrated Point Transformer V3 \cite{wu2024ptv3}, and canonization of the SMILES representation of small molecules \cite{canonicalSMILES}.

In this paper, we provide a theoretical study of canonization from the perspective of generalization. Our focus is both on comparing canonization with other options for enforcing invariance and on comparing the quality of different canonizations. Our main  contributions are threefold:

\begin{enumerate}
\item \textbf{Canonization vs. Group Averaging:} We prove that the generalization bounds of canonized models are at best equal to the bounds of group-averaged models, and at worst equal to the bounds of non-invariant models (Section \ref{sec:gen}).

\item \textbf{Canonizations and continuity:} We show that continuous, isometric canonizations attain the \emph{optimal} generalization error achieved by group averaging, while discontinuous canonizations can lead to \emph{poor} canonizations whose error bounds are identical to non-invariant models (Section \ref{sec:evaluation}).

\item \textbf{Hilbert Canonization:} We prove that
canonizing point clouds with respect to permutation via lexicographical sorting is inferior to the  Hilbert canonizations used in  Point Transformer V3 \cite{wu2024ptv3}, thus establishing the first theoretical justification for the empirical success of this method (Section \ref{sec:case_study}).
\end{enumerate}

\subsection{Related Work}
\paragraph{Canonization vs. Group Averaging} It was observed that randomized SMILES can often outperform canonized SMILES \cite{ArusPous2019,Bjerrum2017SMILESEA} (see also \cite{ito2026random}). Our first contribution gives a theoretical justification for this empirical observation.

\paragraph{Discontinuity} It was shown in \cite{dym2024equivariant,pmlr-v235-baker24a} that in many invariant learning scenarios, continuous canonizations are mathematically impossible. Attempts to mitigate this issue were made in  \cite{tahmasebi2025regularity,lin2026adaptive}. Our second contribution complements these works by showing how discontinuity of canonization hurts the generalization of the canonized model.

\paragraph{Invariance and Generalization} Many papers have considered generalization analysis of invariant models \cite{vasileiou2025survey,petrache2023approximationgeneralization,maskey2025generalization, franks2024weisfeiler}, but they do not consider canonizations. To our knowledge, the only paper considering the generalization of canonizations is \cite{tahmasebi2025generalization}. They prove that in some settings the \emph{approximation error} of canonized models is lower than group averaged models, and as a result the \emph{expected error} of canonized models can be lower when the model is presented with enough samples. In contrast, our first conclusion is that the \emph{generalization error bound} of group average models is always lower than the generalization error bound of canonized models.

\subsection{Notation}
\paragraph{Symmetry space} Throughout the paper, we will consider a metric space $(K,\rho) $, where $K$ is endowed with the action of a group $\G$. We will assume that the metric is $\G$ invariant, i.e., $ \rho(g\cdot x, g \cdot y)=\rho(x,y)$ for all $x,y\in K$ and $g \in \G $. We denote the orbit of $x\in K$ by  $[x]=\{g\cdot x| \quad g\in \G \}$, and denote the quotient space to be the space of orbits $K/\G:=\{[x]|x \in K\}$. On this space, we define a metric $\rho_\G([x],[y]):=\min_{g \in \G}\rho(g\cdot x,y)$, where we explicitly assume that the minimum in the definition of $\rho_\G$  exists, which is the case in most examples of practical interest. This assumption, together with the $\G$ invariance of $\rho$, guarantees that $\rho_{\G}$ is a metric. We call a triplet $(K,\rho,\G)$ satisfying the conditions of this paragraph a \emph{symmetry space}.

\paragraph{Canonization} Let $(K,\rho,\G) $ be a symmetry space. We say that $c:K \rightarrow K$ is a canonization if for every $x\in K$, (i) $c(x)\in [x]$, and (ii) all  $y\in [x]$ satisfy $c(y)=c(x) $. We will sometimes use canon. as an abbreviation for canonization. We also add the assumption that the canonization is Borel measurable.

\textbf{Covering number}
Let $(K,\rho)$ be a metric space, and let $A\subseteq K$. We say that $C\subseteq K$ is an $\epsilon$-cover of $A$ if, for every $x\in A$, there exists $y\in C$ such that $\rho(x,y)\leq\epsilon$. The minimum cardinality of such an $\epsilon$-cover is called the $\epsilon$-covering number of $A$ (with centers in the ambient space $K$), and is denoted by $\N(A,\rho,\epsilon)$.

\textbf{Generalization bounds via covering numbers}
Our discussion of the generalization properties of canonizations and other invariant models is based on a popular framework from \cite{xu_robustness_2012} for bounding generalization via numbers. We begin with a short review of this framework.

Consider the problem of learning an unknown function $f:X\to Y $ from a finite set of $n$ samples $S=\{(x_i,f(x_i)) \mid i=1,\ldots,n \}$ of the function drawn from the distribution $\mu$. Assume that we have some algorithm (e.g., gradient descent) which, given $S$, returns a hypothesis $h_S:X \to Y $. Let $\ell:Y \times Y \to \RR  $ be a function which will be used as a loss function. The \emph{expected loss} $\Lexp$ and the \emph{empirical loss} $\Lemp $ are defined via
$$\Lexp(h_S) \triangleq \mathbb{E}_{x \sim \mu} \ell(h_S(x),f(x)), \quad \Lemp(h_S)\triangleq\frac{1}{n}\sum_{i=1}^n \ell(h_S(x_i),f(x_i))$$

The \emph{generalization} error measures how much the empirical error of $h_S$ on the training data can deviate from its true loss (the expected loss) on the data distribution. It can be bounded  using the Lipschitz constants of the functions $h_S,f,\ell$ and the covering number of the domain $X$, as follows:

\begin{restatable}{theorem}{MainGeneralizationBound}
\label{thm:MainGeneralizationBound}[Proof  in Appendix \ref{app:gen}, based on \cite{xu_robustness_2012} ]
Let $(K,\rho)$ and $(Y,\rho_Y)$ be metric spaces, and assume that $X\subseteq K$ is totally bounded. Let $\mu$ be a Borel probability distribution on $X$.  Let $f:X\to Y$ be a $c_f$-Lipschitz function, let $h_S:X\to Y$ be a $c_{h_S}$-Lipschitz function, and let $\ell:Y\times Y\to[0,M]$ be a $c_\ell$-Lipschitz function. Then for any $0<\epsilon$ and $0<\delta<1$ and natural $n$,  for a training sample set $S$ generated by $n$ IID draws from  $\mu$, with probability of at least $ 1-\delta$,
$$|\Lexp(h_S) - \Lemp(h_S)| \le 2c_\ell (c_{h_S}+c_f)\epsilon+M \sqrt{\frac{2\mathcal{N}( X, \rho,\epsilon)\ln 2 + 2 \ln(1/\delta)}{n}}$$
\end{restatable}

\section{Generalization, Canonization and Group Averaging}\label{sec:gen}
In this section, we consider a task where the function  $f$ to be learned is invariant to a group action. This function is approximated by a 'learned', possibly non-invariant,  Lipschitz function $\tilde h$, which we will call the \emph{backbone model}.  We then consider three strategies of replacing  $\tilde h$ with an alternative model $h$:
\begin{enumerate}
  \item \textbf{Symmetry Agnostic models} Models that do not incorporate symmetry $h=\tilde h$.
  \item \textbf{Canonized models}  Models that incorporate symmetries by canonization, $h= \tilde h \circ c $.
  \item \textbf{Averaged and invariant models} Models that incorporate symmetries by group averaging $h(x)=\int_{\G} \tilde{h}(g\cdot x)dg $, or models where $\tilde h$ is already invariant by construction, and so $h=\tilde h $. In either case, we have that $h$ and $\tilde h$ have the same Lipschitz constant (see Proposition \ref{prop:Lavg} in the appendix).
\end{enumerate}
In this section we establish a heirarchy for the generalization bounds obtained for these three methods. The first step is the following proposition which bounds the generalization error using Theorem \ref{thm:MainGeneralizationBound}:

\begin{restatable}{proposition}{generalizationbound}[Proof  in Appendix \ref{proof_of_three_cases}]
\label{propGen}
Let $(K,\rho,\G)$ be a symmetry space. Let $(Y,\rho_Y) $ be a metric space, and assume $K$ is totally bounded. Let $\mu$ be a Borel probability distribution on $K$. Let $f:K\to Y $ be a $c_f$  Lipschitz function \textbf{which is $\G $ invariant}, and let $\ell:Y\times Y \to [0,M] $ be a $c_\ell $ Lipschitz function , then for any $0<\epsilon$ and $0<\delta<1$ and natural $n$, for a training sample set $S$ generated by $n$ IID draws from $\mu$,  with probability of at least $ 1-\delta$, we have
$$|\Lexp(h_S) - \Lemp(h_S)| \le 2c_\ell (c_{h_S}+c_f)\epsilon+M \sqrt{\frac{2\N(h_S) \ln 2 + 2 \ln(1/\delta)}{n}}$$

where
\begin{equation}\label{eq:differentN}
\N(h_S)=\threepartdef{\N(K,\rho,\epsilon)}{h_S \text{ is } c_{h_S} \text{ Lipschitz }}{\N(c(K),\rho,\epsilon)}{h_S=\tilde{h}_S\circ c, \tilde{h}_S \text { is } c_{h_S} \text{ Lipschitz, and }  c \text{ is a canon.}} {\N(K/\G,\rho_{\G},\epsilon)}{h_S \text{ is } c_{h_S} \text{ Lipschitz and  } \G \text{ invariant }}
\end{equation}
\end{restatable}

\begin{proof}[Proof idea]
The claim for a Lipschitz function $h_S$ without any invariant structure follows immediately from Theorem \ref{thm:MainGeneralizationBound}. For a canonized function  $h_S=\tilde h_{S}\circ c$, the claim follows from applying Theorem \ref{thm:MainGeneralizationBound} to the domain $X=c(K)$ and  the function $\tilde h_{S}$.
The claim for $h_S$ which is both Lipschitz and invariant follows from the fact that due to invariance $h_S, f$ can be identified with functions $\hat h_S, \hat f : K/\G \to Y $ satisfying $\hat h_S([x])=h_S(x), \hat f([x])=f(x), \forall x\in K $, and the Lipschitz constant remains unchanged after the identification.
\end{proof}

For each of the three model types considered in Proposition \ref{propGen}, a different value of $\N(h_S)$ is provided. The following simple lemma establishes the relationship between these three values:
\begin{proposition}
\label{prop:covering_bounds}
    For any symmetry space $(K,\rho,\G)$, and any $\epsilon>0$,  the following inequality holds:
    \begin{equation}
        \label{eq:coveringIneq}
        \N(K/\G,\rho_{\G},\epsilon)\leq \N(c(K),\rho,\epsilon)\leq \N(K,\rho,\epsilon)
    \end{equation}
\end{proposition}

\begin{proof}
The inequality $\N(c(K),\rho,\epsilon)\leq \N(K,\rho,\epsilon)$ is immediate since $c(K)\subseteq K$.
The inequality $\N(K/\G,\rho_{\G},\epsilon)\leq \N(c(K),\rho,\epsilon)$ follows since the quotient map $q(x)=[x]$ is $1$-Lipschitz and $q(c(K))=K/\G$. Hence, the image under $q$ of any $\epsilon$-cover of $c(K)$ is an $\epsilon$-cover of $K/\G$.

\end{proof}
Proposition \ref{propGen} and Proposition \ref{prop:covering_bounds} lead to the following conclusion:
\begin{conclusionbox}{Conclusion 1}
The generalization bounds of canonized models are at best equal to the bounds of invariant and group-averaged models, and at worst equal to the bounds of non-invariant models.
\end{conclusionbox}
Some reservations are in order: this result obtains to the upper bounds in Proposition \ref{propGen}, but these may not be tight. Moreover, the Lipschitz constants obtained by the different methods will not necessarily be identical. Nonetheless, these results lead us to conjecture that in terms of generalization, and overall preformance, invariant and averaged models should outperform canonized models, which in turn will outperform non-invariant methods. In Section \ref{sec:experiments} we give empirical support for this claim, by providing several experiments in which group averaging outperforms canonized models. At the same time, we note that group averaging requires more computational resources and is infeasible for large groups. Invariant models also can lead to more complex and time consuming models. Thus, canonization still can have important advantages in practice. Moreover, the inequality in \eqref{eq:coveringIneq} may not be very large, in which case performance differences may be minor. The remainder of the paper focuses on this gap for different canonizations.

\section{Canonizations: From Poor to Optimal}\label{sec:evaluation}
The inequality in Equation \ref{eq:coveringIneq} gives us a theoretical way to measure the quality of a canonization. Informally, 'good canonizations' are those for which the upper bound is 'almost achieved'. Our goal in this section is to understand to what extent 'good canonizations' are possible. 

We first show that the ratio between the three different bounds in \ref{eq:coveringIneq} is at most the cardinality of the group. Thus, for small groups, the difference between canonization and invariant models is not expected to be very substantial. This is an immediate consequence from the following proposition:
\begin{restatable}{proposition}{cardinality}[See proof in ~\ref{upper_bound}]
\label{prop:cardinality}
Let $(K,\rho,\G)$ be a symmetry space, and assume $\G$ is finite. Then
\[
\N(K,\rho,\epsilon)\le |\G|\cdot \N(K/\G,\rho_{\G},\epsilon).
\]
\end{restatable}

In the next step, we will study in what cases the upper and lower bounds in \eqref{eq:coveringIneq} are achieved. Formally:
\begin{definition}
Let $(K,\rho,\G)$ be a symmetry space and let $\epsilon$ be a positive number. We say that $c$ is $(K,\epsilon) $ \emph{optimal} if it attains the lower bound in \eqref{eq:coveringIneq}, namely
$
\N(c(K),\rho,\epsilon) = \N(K/\G,\rho_{\G},\epsilon).
$
We say that $c$ is $(K,\epsilon) $ \emph{poor}  if it attains the upper bound in \eqref{eq:coveringIneq}, namely
$
\N(c(K),\rho,\epsilon)=\N(K,\rho,\epsilon).
$
\end{definition}
We will next show conditions for optimal and poor canonizations, and show how these conditions relate to the continuity of the canonization.

\textbf{Optimal canonizations:}
First, we prove that canonizations which are isometries are optimal. We say that a canonization $c:K\to K $ is an \emph{isometry} if $\rho(c(x),c(y))=\rho_{\G}([x],[y])  $ for all $x,y\in K$.
\begin{restatable}{proposition}{isometry}\label{prop:isometry}[proof in \ref{Proof_isometry}]
    Any isometric canonization $c:K\to K$ is $(K,\epsilon) $ optimal for all $\epsilon>0$.
\end{restatable}

We next use this claim to give several examples of  optimal canonizations:
\begin{restatable}{example}{isometries}\label{prop:isometries}
    The following canonizations are isometries (see proof in \ref{Proof_isometries})
\begin{enumerate}
\item The canon. $c(x)=|x|$ with respect to the action of $\{-1,1\} $ on $\R$ and the standard metric.
\item The canon. $\mathrm{sort}:\RR^n \to \RR^n $ with respect to the action of permutations and a metric $\rho$ induced by a $p$ norm.
\item The centralization canon. defined for point clouds in $\RR^{d\times n}$ with respect to the action of translation by vectors in $\RR^d$, with the metric $\rho$ induced by the Frobenius norm.
\end{enumerate}
\end{restatable}

In Proposition \ref{prop:isometryLip} in the appendix, we prove that a canonization is isometric if and only if it is $1$-Lipschitz in the standard sense.  As a result, in the setting where there can be no continuous canonizations, such as those discussed in  \cite{dym2024equivariant,pmlr-v235-baker24a}, there can be no isometric canonizations.

\textbf{Poor canonizations:} We next show that discontinuities in canonizations can lead to poor canonizations. This occurs under two conditions: firstly, the canonization is `strongly discontinuous' in the sense that it has $|\G|$ partial limits. Secondly, the domain  of the group action is chosen adversarially to exploit the discontinuity of the canonization.
\begin{restatable}{proposition}{poorcanonization}[proof in \ref{poor_proof}]
\label{prop:poor}
Let $(K,\rho,\G)$ be a symmetry space, where $\G$ is a finite group. Let $c$ be a canonization that has $|\G|$ partial limits at a point $x\in K$. Then, for any sufficiently small $\epsilon>0$, there exists a $\G$-invariant set $A\subseteq K$ such that

$$
\N(c(A),\rho,\epsilon)
=
\N(A,\rho,\epsilon)
=
|\G|\cdot \N(A/\G,\rho_{\G},\epsilon).
$$

\end{restatable}

Based on this proposition, we give two examples of poor canonizations. 

\begin{example}
\label{ex:sign_poor}
Consider the action of $\G=\{-1,1\}$ on $\R $. Consider the canonizations
$$c_1(t)=\twopartdef{|t|}{|t|>1/2}{-|t|}{|t|\leq 1/2}, \quad  c_\infty(t)=\twopartdef{|t|}{t\in \mathbb{Q}}{-|t|}{t\not \in \mathbb{Q}} $$
visualized in Figure \ref{fig:1d_canonizations}. Both of these canonizations are discontinuous and have a point with $|\G|=2 $ partial limits. Therefore, by  Proposition \ref{prop:poor} for every small enough $\epsilon>0$ we can choose an adversarial domain for which the canonization is poor. For $c_1$ we can choose for every $\epsilon\in (0,\frac{1}{2})$ the set  $K=B_\epsilon(-\frac{1}{2})\cup B_\epsilon(\frac{1}{2})$ and then $\N(K,\rho,\epsilon)=\N(c_1(K),\rho,\epsilon)=2$ and so the canonization is poor.
In the canonization $c_\infty$ the discontinuities are much more extreme. In this case, we can choose a non-adversarial domain $K=[-1,1]$, and the canonization will be poor for any $\epsilon>0$ because the covering number of a set and its closure are the same, and in our case, the closure of $c_{\infty}(K)$ will be all of $K$.
\end{example}

\begin{figure}[h]
    \centering
      \caption{\it Visualization of three canonizations for the action of $\G=\{-1,1\}$ on $\mathbb{R}$: the natural  canonization $c(x)=|x|$, which is shown to be optimal (Example \ref{prop:isometries}), and two discontinuous canonizations, which on an appropriate domain are poor (Example \ref{ex:sign_poor}).}
    \label{fig:1d_canonizations}
    \includegraphics[width=0.8\linewidth]{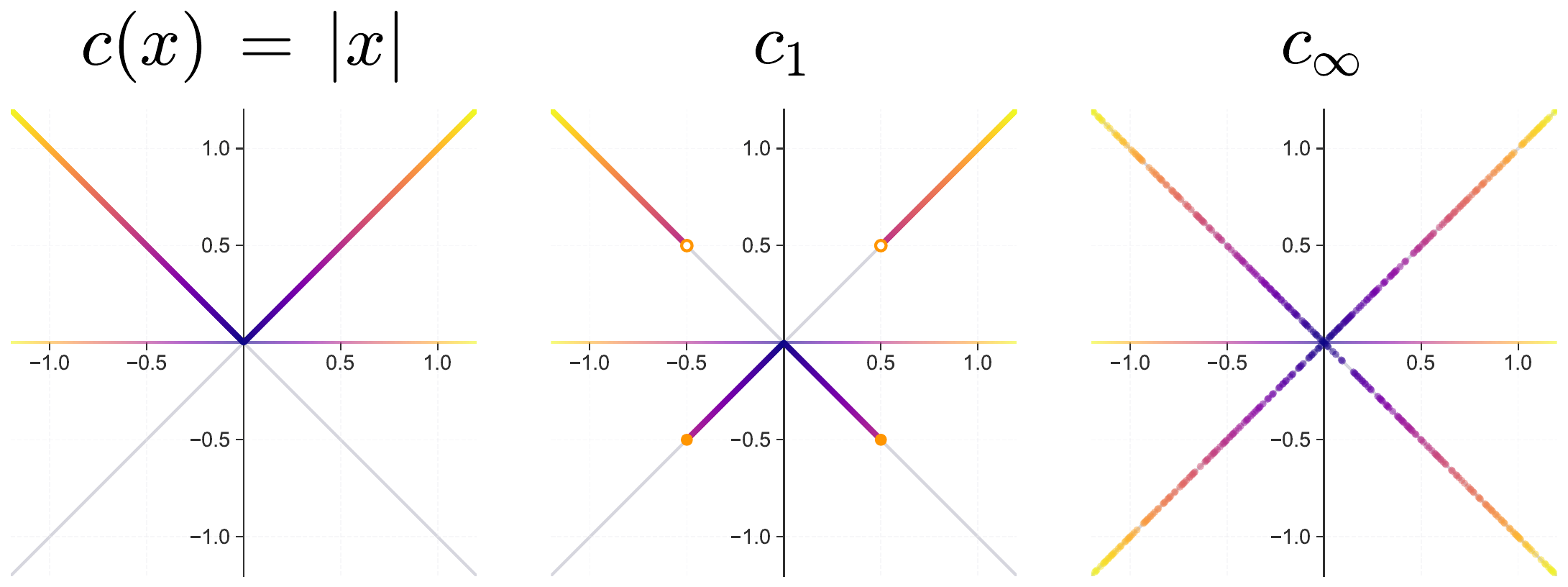}
\end{figure}

We next present the second example.

\begin{example}\label{ex:multisort}
Consider the action of the permutation group $\G=S_n$ on the space of matrices $\RR^{d\times n}$ by permuting the columns of the matrices. We call such matrices `point clouds'. We assume that $n,d\geq 2$ (in the case $d=1$, sorting is an isometry by Example \ref{prop:isometries}). In this case, it is known that no continuous canonization exists \cite{dym2024equivariant}. A natural (discontinuous) canonization in this scenario is lexicographical sorting $\csort$, which permutes the columns of the matrices so that the first row of the matrix is sorted from small to large, and in the event of ties in the first row, sorts according to the second row, and continues in this fashion until a permutation is uniquely defined. This canonization has $n!=|\G|$ partial limits at any matrix  $M\in \RR^{d \times n} $ satisfying that the matrix does not contain two identical columns, but $M_{1,1}=M_{1,2}=\ldots=M_{1,n}$. Thus, according to Proposition \ref{prop:poor}, for  small enough $\epsilon>0$ there exists an adversarial domain for which the canonization is poor.
\end{example}
\begin{conclusionbox}{Conclusion 2}
We show that isometric canonizations are optimal, while discontinuous canonizations can yield poor canonizations. These results give a rigorous justification for the importance of continuous canonizations.
\end{conclusionbox}

\section{Case study: Permutation Canonizations}\label{sec:case_study}

In Example \ref{ex:multisort}, we showed that lexicographical sorting is a poor canonization on an adversarial domain. In this section, we analyze the covering numbers of lexicographical sorting when considering a fixed, `natural' domain of point clouds $K=[0,1]^{d\times n}$. Throughout this section, we will consider the action of the group $\G=S_n$, and the distance induced by the element-wise infinity norm $\rho_\infty(X,Y)=\|X-Y\|_\infty $. We will show that (for fixed $\epsilon,d$) the covering number of Lexsort grows exponentially in $n$, the number of points in the point cloud, while the covering number of the quotient space only grows polynomially in $n$. Moreover, we will consider an alternative canonization based on Hilbert curves, and show that its covering numbers also grow only polynomially in $n$. The significance of this result is that it gives a theoretical justification for the choice of  Point Transformer V3 \cite{wu2024ptv3} to use Hilbert canonizations for point cloud serialization.

We begin by proving the exponential growth of Lexsort canonizations:

\begin{restatable}{lemma}{lexsortcovering}[Proof in \ref{lexsortcovering}]
\label{lem:lexsortcovering}
Let $k,d,n\geq 2$ be natural numbers, and set $\epsilon=1/(2k), K=[0,1]^{d\times n}$ and $\G=S_n$. Then the covering number of the lexicographical sorting canonization satisfies
\begin{equation}
\label{eq:exp}
\left( \frac{1}{2\epsilon}\right)^{(d-1)\cdot n + 1}\leq \N(\csort(K),\rho_\infty,\epsilon).
\end{equation}
\end{restatable}
\begin{proof}[Proof idea]
The Lexsort canonization is ineffective on the set of its `severe discontinuities'- point clouds whose first row is constant. These point clouds form a $(d-1)\cdot n+1 $ dimensional hypercube. The covering number of $\csort(K)$ is lower bounded by the covering number of this hypercube, which is the left-hand side of \eqref{eq:exp}.
\end{proof}
Next, we prove that for fixed $\epsilon,d$,  the quotient space's covering number grows polynomially in $n$:
\begin{restatable}{lemma}{quotientcovering}[Proof in \ref{quotient_cover}]
\label{lem:quotient_covering}
Let $n,d$ be natural numbers and $\epsilon\in (0,1)$. Set $K=[0,1]^{d\times n}$ and $\G=S_n$. Then the covering number of the quotient space satisfies
\begin{equation}
\label{eq:combinatorial}
\N(K/\G,\rho_{\G},\epsilon)
\leq
\binom{n+k^{d}-1}{n},
\text{ where }
k=\left\lceil \frac{1}{2\epsilon} \right\rceil .
\end{equation}
\end{restatable}
\begin{proof}[Proof idea]
Set $k=\lceil 1/(2\epsilon) \rceil $. Partition $[0,1]^d$ into $k^d$ hypercubes of side length $1/k$ and radius $1/(2k)\leq \epsilon$. Using their centers for each of the $n$ points yields an $\epsilon$-cover of $K=[0,1]^{d\times n}$. Let $C$ denote the collection of all point clouds of $n$ points whose points are centers of these hypercubes. The image of $C$ under the quotient map is  an $\epsilon$ cover $q(C)$ of $K/\G$. Note that points in $C$ which are permutations of each other are mapped by $q$ to the same orbit. Therefore, the cardinality of the cover $q(C)$ is the number of equivalence classes with respect to the relation on $C$ defined by the group: $c_1 \sim c_2 \iff \exists \tau \in S_n : c_1 = \tau\cdot c_2$. As explained in more detail in the proof, the number of equivalence classes is the combinatorial expression in \eqref{eq:combinatorial}.
\end{proof}

\subsection{Hilbert Curve Canonization}
Finally, we will provide an intuitive explanation of the Hilbert canonization and formulate the claim that its covering number exhibits only polynomial growth with respect to the number of points in the point cloud. For a full definition of the canonization and proof, see Appendix \ref{app:hilbert}.
\begin{figure}[h]
    \centering
        \caption{\it The Hilbert curves $H_m$ for $d=2$ and $m=1,2,3$ induce an ordering on a grid in $[0,1]^d $ with $2^{md}$ points. On the right, we see the discontinuous Lexsort ordering on the grid.}
    \includegraphics[width=\linewidth]{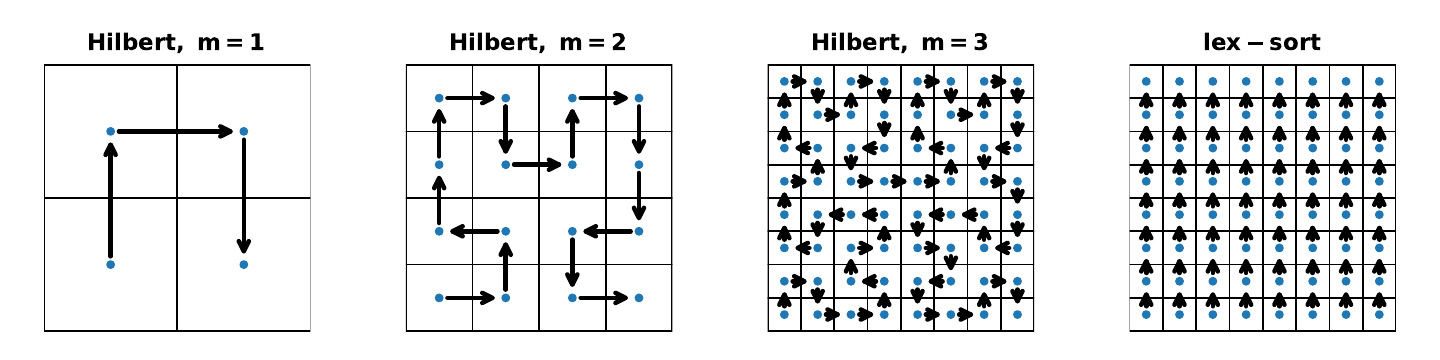}
    \label{fig:hilbert}
\end{figure}

The Hilbert curve is a space-filling curve, mapping the unit interval $[0,1]$ continuously onto a hypercube $[0,1]^d $. It is defined as a limit of curves $H_m:[0,1]\to [0,1]^d $ which are not onto but whose image becomes more and more `dense' in $[0,1]^d $ as $m$ grows, as shown in Figure \ref{fig:hilbert} for the case where $d=2$. The curves $H_m$ can be used to define an order on $[0,1]^d$ (at least on points in the image of the curve) and thus induce a canonization.
Intuitively, the advantage of this order over the Lexsort ordering is that adjacently ordered points  are mapped to nearby points in the ambient space, due to the  continuity of the Hilbert curve (see Figure \ref{fig:hilbert}).

However, while previous work \cite{wu2024ptv3,hilbertnet} cited this continuity as justification for  Hilbert ordering, a formal justification of why this continuity is important for learning has not yet been provided. We address this caveat by showing the polynomial growth of the covering number of Hilbert canonizations:
\begin{restatable}{theorem}{hilbert}[Proof in \ref{hilbert_proof}]
\label{thm:hilbert}
Let $d,n,m\geq 2$ be natural numbers, let $\epsilon\in (0,1)$, and assume that $\epsilon>2^{-m-1} $. Set $ K=[0,1]^{d\times n}$ and $\G=S_n$.    Then the covering number of the Hilbert canonization $c_m$ satisfies
$$\N(c_m(K),\rho_\infty,\epsilon)\leq  \binom{n+\lceil 1/ (2\delta) \rceil-1}{n},   \text{ where } \delta= \left(\frac{\epsilon-2^{-m-1}}{4}\right)^d$$
\end{restatable}
\begin{proof}[Proof idea]
When restricting to $n$-tuples in $K$ which are all in the image of $H_m$, the Hilbert canonization can be written as $H_m \circ \mathrm{sort} \circ H_m^{-1} $. Since $\mathrm{sort}$ (Example \ref{prop:isometries}) is an isometry, the image of $  \mathrm{sort} \circ H_m^{-1}$  can be covered efficiently. This cover can then be pushed forward to cover the image of   $H_m\circ \mathrm{sort} \circ H_m^{-1}$, using the fact that $H_m$ is Hölder continuous.
\end{proof}

 Theorem \ref{thm:hilbert} shows that the covering number of the Hilbert canonization grows at worst polynomially in $n$, for fixed $\epsilon>0$. This is in contrast with the exponential growth of the sorting canonization, proving a substantial gap between them.

\begin{conclusionbox}{Conclusion 3}
 Our analysis proves that the covering number of Hilbert canonizations is far superior to Lexsort canonization, giving a theoretical justification for its superior performance in practice.
\end{conclusionbox}

\section{Experiments}
\label{sec:experiments}
We conduct several experiments to verify our theoretical insights and their manifestation in practical learning scenarios. We focus on validating the hierarchy between group averaging, canonization, and non-invariant models (Conclusion 1), and on the hierarchy between Hilbert and Lexsort canonizations (Conclusion 3).

\paragraph{Coverage on ModelNet} In our first experiment, we considered the ModelNet \cite{wu20153d} point cloud classification task. Our first goal was to examine whether the covering-number hierarchy predicted by our theory is reflected in the geometry of real ModelNet point clouds, when comparing the quotient metric, Hilbert and Lexsort canonizations, and the original Euclidean metric.

As we do not have access to the `true point cloud' domain $K$, but only to the train and test sets, we do not compute covering numbers, but a related notion we call the coverage. For given sets $T$ (representing the test set) and $R$ (the train set), we define the coverage $\epsilon(R,T)$  to be the smallest $\epsilon$ such that $R$ is an $\epsilon$ cover of $T$. We compute the coverage by choosing, for each test sample $t\in T$, the minimum distance from $t$ to all training samples (for computational efficiency, we consider only samples with the same label as $t$). We call this number $q_t$. The coverage is then given by $\epsilon(R,T)=\max_{t\in T} q_t $. We call this measure the coverage or \emph{max coverage}. We also consider a more robust measure, the \emph{mean coverage} $\epsilon_{\mathrm{mean}}(R,T)=\mathrm{mean} \{ q_t| t\in \mathrm{Test} \} $.

\begin{table}[h]
\centering
\caption{\it Estimates of covering number on ModelNet via mean and max coverage, when using Euclidean distances, Lexsort or Hilbert canonizations, and the group distance.}
\label{tab:distances_both}
\begin{tabular}{lcccc}
\toprule
 & \multicolumn{2}{c}{\textbf{Mean coverage} $\downarrow$}
 & \multicolumn{2}{c}{\textbf{Max coverage} $\downarrow$} \\
\textbf{Method}
& \textbf{ModelNet10}
& \textbf{ModelNet40}
& \textbf{ModelNet10}
& \textbf{ModelNet40} \\
\midrule
Euclidean distance & 0.4545 & 0.3823 & 0.629  & 0.7304 \\
Lexsort distance    & 0.3770 & 0.2989 & 0.5407 & 0.5550 \\
Hilbert distance & 0.2013 & 0.1951 & 0.5283 & 0.5490 \\
Group distance   & \textbf{0.0917} & \textbf{0.0958} & \textbf{0.2742} & \textbf{0.4880} \\
\bottomrule
\end{tabular}
\end{table}

We compute the max and mean coverage on ModelNet10 and ModelNet40, using four different distances: (i) the mean pointwise Euclidean distance,  the Euclidean distance after applying (ii) Lexsort or (iii) Hilbert canonizations to the train and test sets, and (iv) the quotient metric $\rho_\G $ (in this case, the Wasserstein distance). The results, shown in Table \ref{tab:distances_both}, show that the theoretical covering number hierarchy we derived is also preserved when considering coverage of real data: the coverage of  $\rho_\G $ is the lowest, followed by Hilbert, Lexsort, and finally standard Euclidean distance.

\paragraph{Classification on ModelNet}
Next, we examine whether the hierarchy in covering numbers between the different methods translates into a corresponding hierarchy in accuracy and generalization error. We use a simple, purely non-invariant backbone, a Global MLP, and compare its performance (i) without canonization, (ii) with Lexsort canonization, and (iii) with Hilbert canonization at resolution $12$. We additionally evaluate DeepSets as an invariant Lipschitz model (exact group averaging is computationally infeasible here due to the prohibitively large permutation group). The results, shown in Table~\ref{tab:best_modelnet10and40}, reveal the expected hierarchy: Hilbert outperforms Lexsort, which in turn outperforms the plain MLP, while the invariant DeepSets model achieves the strongest overall performance. The generalization error column (defined as Train Acc.~$-$~Test Acc.) indicates that these performance differences are largely due to improved generalization.

\begin{table}[h]
\centering
\caption{\it On the ModelNet classification task, we compare a vanilla MLP without canonization, the same architecture with Lexsort and Hilbert canonization, and DeepSets as an invariant Lipschitz model. As predicted by our theory, improved invariance leads to better generalization and overall accuracy.}
\label{tab:best_modelnet10and40}
\begin{tabular}{lcccc}
\hline
 & \multicolumn{2}{c}{ModelNet10} & \multicolumn{2}{c}{ModelNet40} \\
Model & Test Acc $\uparrow$ & Gen. error $\downarrow$ & Test Acc $\uparrow$ & Gen. error $\downarrow$ \\
\hline
MLP (no canon.) & 0.573 & 0.427 & 0.441 & 0.559 \\
MLP (Lexsort)   & 0.762 & 0.215 & 0.692 & 0.256 \\
MLP (Hilbert)   & 0.782 & 0.116 & 0.742 & 0.082 \\
DeepSets        & \textbf{0.813} & \textbf{0.077} & \textbf{0.755} & \textbf{0.039} \\
\hline
\end{tabular}
\end{table}

To further verify the generality of our findings, we also trained the three Global MLP alternatives, together with DeepSets, on smaller samples from the ModelNet40 dataset. We find that the $\text{DeepSets}>\text{Hilbert}>\text{Lexsort}>\text{no canon.}$ hierarchy is preserved across all training-set sizes, as shown in Table~\ref{tab:data_scarcity} in the appendix.

We note that this experiment is only illustrative; the results attained by simple vanilla MLPs, even with Hilbert canonization, are far from state-of-the-art. In the context of our theory, this is consistent with the fact that successful models for this task are typically permutation invariant, and hence their generalization is governed by the covering number of the quotient space, which is always no larger than the covering number induced by canonization (Conclusion~1).

We next conduct further experiments to corroborate Conclusion~1, which predicts that, in terms of generalization, invariant models or group-averaged predictors should enjoy an advantage over canonized models, which in turn should improve over models that ignore the underlying invariance.

\textbf{Group averaging on PCA frames} Our first experiment is still with the ModelNet40 classification task, but while the previous experiments investigate canonization w.r.t. permutations, and implicitly rely on the fact that ModelNet models are aligned w.r.t rotations, in this experiment we use the  DeepSets \cite{zaheer2017deep} model to deal with permutation invariance, and will focus on rotation invariance: we will assume the rotational degree of freedom of point clouds in ModelNet is initially unknown, and is fixed using the PCA axes as suggested in \cite{punyframe}. Assuming the point cloud singular values are pairwise distinct, this procedure determines the point cloud orientation uniquely up to a $\{-1,1\}^3$ symmetry.

We consider four ways to deal with this symmetry: (i) Pure PCA, where the symmetry is ignored (ii) Skewness canonization, where the global sign of the $x,y$ and $z$ coordinates is chosen so that their third moment is positive (iii) group averaging and (iv) group augmentation where in each epoch a random group element is applied to the input. This  can be seen as an approximation of group averaging. The results of this experiment, detailed in Table \ref{tab:pca_canonization}  demonstrate a clear performance hierarchy. Group averaging achieves the highest overall accuracy, while group augmentation and skewness canonization yield comparable final results, with all three methods significantly outperforming the pure PCA baseline. We note that, as highlighted in Figure \ref{fig:loss_pca} in the appendix, the group augmentation approach suffers from a noticeably slower convergence rate compared to the other three methods.

\textbf{Group averaging for image classification} Next, we check our approach on image classification. We take the rotated MNIST \cite{larochelle2007empirical} dataset, and consider the action of the $p4$ group, which is the four-element group generated by $90$ degree rotations. We consider 4 approaches: simple CNN, the canonization from \cite{kaba2023equivariance} with random or learned weights (named CN(p4 frozen)-CNN and CN(p4)-CNN, respectively), and a group average over the $4$ elements of the group. As shown in Table \ref{tab:mnist}, the canonizations outperform the simple CNN, but are outperformed by group averaging, as predicted in Conclusion 1. We also see in the table that the accuracy is correlated with the mean coverage of the different methods.

\begin{table}[h]
    \centering
    \setlength{\tabcolsep}{3pt}

    \begin{minipage}[t]{0.48\linewidth}
        \centering
        \caption{\it Test accuracy on ModelNet40 with a DeepSets model with
        orientation determined by PCA, and sign ambiguity handled via averaging,
        sign augmentation, canonization, or Pure PCA (ignoring the ambiguity).}
        \label{tab:pca_canonization}

        \begin{tabular}{lc}
        \toprule
        \textbf{Method} & \textbf{Accuracy} $\uparrow$ \\
        \midrule
        Sign Averaging       & \textbf{0.760} $\pm$ 0.003 \\
        Sign Augmentation    & 0.742 $\pm$ 0.002 \\
        Skewness Canonization & 0.742 $\pm$ 0.005 \\
        Pure PCA             & 0.729 $\pm$ 0.006 \\
        \bottomrule
        \end{tabular}
    \end{minipage}%
    \hfill%
    \begin{minipage}[t]{0.48\linewidth}
        \centering
        \caption{\it Accuracy and mean coverage on Rotated-MNIST, which is
        invariant with respect to the group $p4$ of $90$ rotations. We consider
        a non-invariant CNN, two canonizations from
        \cite{kaba2023equivariance}, and group averaging.}
        \label{tab:mnist}

        \begin{tabular}{lcc}
        \toprule
        \textbf{Model} & \textbf{Test Acc.} $\uparrow$ &
        \textbf{Coverage} $\downarrow$ \\
        \midrule
         CNN               & $94.7 \pm 0.12$          & 0.185 \\
        CN(p4)-CNN        & $96.5 \pm 0.13$      & 0.175 \\
        CN(p4 frozen)-CNN & $95.5 \pm 0.15$        & 0.176 \\
        AvgCNN            & $\textbf{97.3} \pm 0.12$ & \textbf{0.167} \\
        \bottomrule
        \end{tabular}
    \end{minipage}
\end{table}

\paragraph{Canonization in deep weight spaces.}
We evaluate canonization on INR classification, where each image is represented by the weights of an implicit neural representation (INR), and the task is to predict the image class directly from these weights~\citep{pmlr-v202-navon23a}. We compare a simple MLP, a canonized MLP, and DWSNet, a specialized permutation-invariant weight-space model~\citep{pmlr-v202-navon23a}. We consider the MNIST-INR and Fashion-MNIST-INR datasets~\citep{pmlr-v202-navon23a}; details of our canonization procedure are given in Appendix~\ref{app:dws_canonization}.

For completeness, we reran the original code for the MLP and DWSNet baselines and report the results obtained from our runs. The results are shown in Table~\ref{tab:dws_canonization}. We observe that DWSNet outperforms the canonized MLP, which in turn outperforms the plain MLP, as predicted by Conclusion~1.

\begin{table}[h]
\centering
\caption{\it INR classification accuracy on MNIST and Fashion-MNIST.
We compare a non-invariant MLP,  a canonized MLP, and DWSNet, a specialized
permutation-invariant weight-space architecture.}
\label{tab:dws_canonization}
\begin{tabular}{lcc}
\toprule
\textbf{Method}
& \textbf{MNIST-INR} $\uparrow$
& \textbf{Fashion-MNIST-INR} $\uparrow$ \\
\midrule
MLP
& $16.34 \pm 0.61$
& $14.41 \pm 0.39$ \\

MLP + Canonization
& $31.06 \pm 0.46$

& $40.02 \pm 0.23$
 \\
DWSNet
& $\textbf{86.66} \pm 0.73$
& $\textbf{65.53} \pm 0.29$ \\
\bottomrule
\end{tabular}
\end{table}

\textbf{Group Averaging for Spectral Embeddings} In the Appendix
 Section \ref{sec:add_exp} we present an additional experiment showing the advantage of averaging over previously proposed canonizations, in the context of sign ambiguity present in spectral embeddings for graph learning.

\section{Conclusion}
In this paper, we establish a methodology for comparing different canonizations with group averaging
and other invariant models by studying their covering number. We show the covering number of
canonizations is at best the covering number of the quotient space (for isometric canonizations), and
at worst the covering number of the original domain (for certain discontinuous canonizations). These
results suggest that group-averaged models are preferable when computationally feasible. We also
show that the covering number of Hilbert canonizations for point clouds is substantially better than
that of Lexsort canonizations.
\paragraph{Limitations and Future Work} A limitation of our analysis is that it relies on upper bounds on
the generalization error (Proposition 2.2), which may not be tight. Nonetheless, our experiments
indicate that lower covering numbers do correlate with better generalization and learning. Looking
towards future developments, a central contribution of this work is the demonstration that a rigorous
theoretical comparison between competing canonization methods can be established via their covering
numbers. This framework complements the traditional focus on a model’s ability to distinguish
between complex symmetric examples, instead highlighting properties that arguably exert a more
significant influence on the ultimate success of the learning process. In future work, we intend
to apply our framework to other canonizations and extend it to enable a rigorous treatment of the
generalization of frame averaging.
\section{Acknowledgments}
All authors acknowledge are funded by the Israeli Science Foundation grant no.272/23 and DFG-ISF grant no. 3597/26. N.D. and S.H. wish to thank
Jonathan W Siegel and Hannah Lawrence for many interesting discussions related to the topics
presented in the paper.

\newpage
\ConferenceBibliography

\newpage
\appendix

\section{Proofs}
In this section, we provide  proofs and claims that are missing in the main text.

\subsection{Generalization and Covering Numbers}\label{app:gen}
Our generalization results are based on the results of \cite{xu_robustness_2012}: a combination of Theorem 1 and Example 4 from  \cite{xu_robustness_2012} yields the following theorem:

\begin{theorem}[Based on \cite{xu_robustness_2012}]\label{thm:xu}
Let $(K,\rho)$ be a metric space, and let $X\subseteq K$ be totally bounded. Suppose that
$0\le l(h_S,x)\le M$ for all $x\in X$, and that $l(h_S,\cdot)$ is Lipschitz continuous on $X$ with Lipschitz constant $c(S)$, i.e.,
\begin{equation*}
|l(h_S,x_1)-l(h_S,x_2)|
\le c(S)\rho(x_1,x_2),
\quad \forall x_1,x_2\in X.
\end{equation*}
Assume that the training sample set $S$ is generated by $n$ IID draws from a Borel probability distribution $\mu$ on $X$. Then, for any $\delta,\epsilon>0$, with probability at least $1-\delta$,
\begin{equation}\label{eq:intermediate}
|\Lexp(h_S)-\Lemp(h_S)|
\le
2c(S)\epsilon
+
M\sqrt{
\frac{
2\mathcal{N}(X,\rho,\epsilon)\ln 2
+
2\ln(1/\delta)
}{n}
}.
\end{equation}
\end{theorem}

We now explain how Theorem \ref{thm:MainGeneralizationBound} from the main text follows from Theorem \ref{thm:xu}. We recall Theorem \ref{thm:MainGeneralizationBound}:
\MainGeneralizationBound*
\begin{proof}
We will apply Theorem \ref{thm:xu}, where to apply it to
the setting discussed in Theorem \ref{thm:MainGeneralizationBound}, we set
$$l(h_S,x)=\ell(h_S(x),f(x))$$

We equip $Y\times Y$ with the product metric
\[
\rho_{Y\times Y}\big((y_1,y_2),(y_1',y_2')\big)
=
\rho_Y(y_1,y_1')+\rho_Y(y_2,y_2').
\]

We note that $l$ is Lipschitz with a Lipschitz constant of $c_\ell (c_{h_S}+c_f)$ as the following shows:
\begin{align*}
|l(h_S,x)-l(h_S,x')|&=|\ell(h_S(x),f(x))-\ell(h_S(x'),f(x'))|\\
&\leq c_\ell \left( \rho_Y(h_S(x),h_S(x'))+\rho_Y(f(x),f(x'))  \right)\\
&\leq c_\ell \left( c_{h_S}\rho(x,x')+c_f\cdot \rho(x,x')  \right)\\
& = c_\ell (c_{h_S}+c_f)( \rho(x,x'))\\
\end{align*}

Now, to conclude the proof, choose any $\epsilon,\delta>0$. Since $l$ is $c_\ell (c_{h_S}+c_f)$ Lipschitz, we have
\begin{align*}
|\Lexp(h_S) - \Lemp(h_S)|
&\stackrel{\eqref{eq:intermediate}}{\le} 2c_\ell (c_{h_S}+c_f) \epsilon + M \sqrt{\frac{2\mathcal{N}( X, \rho, \epsilon) \ln 2 + 2 \ln(1/\delta)}{n}}\\
\end{align*}
which proves Theorem \ref{thm:MainGeneralizationBound}.
\end{proof}

Next, we prove Proposition \ref{propGen}.
\generalizationbound*
\begin{proof}
\label{proof_of_three_cases}
The first case where $h_S$ is Lipschitz (but not invariant) follows directly from Theorem \ref{thm:MainGeneralizationBound}.
Next, we  consider the second case $h_S=\tilde{h}_{S}\circ c$, where $\tilde{h}_{S}$ is $c_{h_S}$ Lipschitz. We note that since $\tilde{h}_S$ operates on $c(S)$, we can think of it as defined by the set $\tilde S=c(S) $, and so we can write $\tilde{h}_S:=\tilde{h}_{\tilde S}$. This notation is more convenient for the proof.

We want to consider the domain to be $c(K)$, so consider the following distribution measure on $c(K)$:
    \begin{align*}
        \tilde{\mu}(A)= \mu(c^{-1}(A))
    \end{align*}
    Note that $$\forall x \in c(K), h_S(x)=\tilde{h}_{\tilde{S}}(x) $$
    Next, for any fixed $n$ samples from $K$, which we denote by $S=\{x_1,x_2,\ldots,x_n\}$, denote its image by $\tilde S=\{c(x_1),\ldots,c(x_n)\}$.  Then
    \begin{align*}
        \Lemp(h_S)=\frac{1}{n}\sum_{i=1}^n \ell(\tilde h_{\tilde S}(c(x_i)),f(x_i)) = \frac{1}{n}\sum_{i=1}^n \ell(\tilde h_{\tilde S}(c(x_i)),f(c(x_i))) = \Lemp(\tilde{h}_{\tilde S})
    \end{align*}
    So, the empirical losses are the same, and next we show the expected losses are the same:
    \begin{align*}
        \Lexp(\tilde{h}_{\tilde S}) &= \mathbb{E}_{x \sim \tilde{\mu}} \ell(\tilde{h}_{\tilde S}(x),f(x)) = \int_{c(K)} \ell(\tilde{h}_{\tilde S}(x),f(x)) d\tilde{\mu} = \int_{K} \ell(\tilde{h}_{\tilde S}(c(x)),f(c(x))) d\mu  \\
        &=\int_{K} \ell(h_S(x),f(x)) d\mu = \Lexp(h_S)
    \end{align*}
    Since $c(K)\subseteq K$ and $K$ is totally bounded, $c(K)$ is totally bounded. We may therefore apply Theorem \ref{thm:xu} with $X=c(K)$, ambient space $K$, and distribution $\tilde{\mu}$. Since $\tilde{h}_{\tilde S}$ is $c_{h_S}$-Lipschitz, with probability at least $1-\delta$ under $\tilde{\mu}$, we obtain:

    \begin{align*}
        |\Lexp(\tilde{h}_{\tilde S}) - \Lemp(\tilde{h}_{\tilde S})| \le 2c_\ell (c_{h_S}+c_f)\epsilon+M \sqrt{\frac{2\mathcal{N}( c(K), \rho,\epsilon)\ln 2 + 2 \ln(1/\delta)}{n}}
    \end{align*}
    Let us abbreviate the right hand side of the equation above by $\eta=2c_\ell (c_{h_S}+c_f)\epsilon+M \sqrt{\frac{2\mathcal{N}( c(K), \rho,\epsilon)\ln 2 + 2 \ln(1/\delta)}{n}}$. We then have
\begin{align*}
1-\delta
&\leq
\mathbb{P}_{\tilde{\mu}^n}
\left(
\left|
\Lexp(\tilde h_{\tilde S})
-
\Lemp(\tilde h_{\tilde S})
\right|
\leq \eta
\right)\\
&=
\mathbb{P}_{\mu^n}
\left(
\left|
\Lexp(\tilde h_{\tilde S}\circ c)
-
\Lemp(\tilde h_{\tilde S}\circ c)
\right|
\leq \eta
\right)\\
&=
\mathbb{P}_{\mu^n}
\left(
\left|
\Lexp(h_S)
-
\Lemp(h_S)
\right|
\leq \eta
\right).
\end{align*}

which concludes the proof of the second case.

In the third case, where $h_S$ is $c_{h_S}$-Lipschitz and invariant, we consider the domain $K/\G$ and define $\pi:K\to K/\G$ by $\pi(x)=[x]$. Define the following distribution on $K/\G$:
\begin{align*}
    \tilde{\mu}(A)=\mu(\pi^{-1}(A)).
\end{align*}

Define $\tilde S=\pi(S)$ and
\[
\tilde{h}_{\tilde S},\tilde{f}:K/\G\to Y
\]
by
\[
\tilde{h}_{\tilde S}([x])=h_S(x),
\qquad
\tilde{f}([x])=f(x).
\]
Note that by invariance, these functions are well defined.

Next, for any sample set $S=\{x_1,\ldots,x_n\}$ from $K$, let
\[
\tilde S=\{\pi(x_1),\ldots,\pi(x_n)\}.
\]
Then
\begin{align*}
\Lemp(h_S)
&=
\frac{1}{n}\sum_{i=1}^n \ell(h_S(x_i),f(x_i))\\
&=
\frac{1}{n}\sum_{i=1}^n
\ell\bigl(\tilde{h}_{\tilde S}(\pi(x_i)),
          \tilde{f}(\pi(x_i))\bigr)\\
&=
\Lemp(\tilde{h}_{\tilde S}).
\end{align*}

Thus, the empirical losses are the same. We next show that the expected losses are also equal:
\begin{align*}
\Lexp(\tilde{h}_{\tilde S})
&=
\mathbb{E}_{x\sim\tilde{\mu}}
\ell(\tilde{h}_{\tilde S}(x),\tilde{f}(x))\\
&=
\int_{K/\G}
\ell(\tilde{h}_{\tilde S}(x),\tilde{f}(x))
\,d\tilde{\mu}(x)\\
&=
\int_K
\ell(\tilde{h}_{\tilde S}(\pi(x)),
     \tilde{f}(\pi(x)))
\,d\mu(x)\\
&=
\int_K
\ell(h_S(x),f(x))
\,d\mu(x)\\
&=
\Lexp(h_S).
\end{align*}

Since $K$ is totally bounded and the quotient map
$\pi:K\to K/\G$ is $1$-Lipschitz and surjective, the quotient space
$K/\G$ is also totally bounded. Moreover, $\tilde{f}$ is
$c_f$-Lipschitz and $\tilde{h}_{\tilde S}$ is
$c_{h_S}$-Lipschitz (see Lemma 20 in
\cite{siegel2026quantitativeapproximationratesgroup}).

Therefore, all the assumptions of Theorem
\ref{thm:MainGeneralizationBound} are satisfied, and applying it to
$\tilde{h}_{\tilde S}$ and $\tilde{f}$ on $K/\G$, we obtain that with
probability at least $1-\delta$,
\begin{align*}
\left|
\Lexp(\tilde{h}_{\tilde S})
-
\Lemp(\tilde{h}_{\tilde S})
\right|
\leq
2c_\ell(c_{h_S}+c_f)\epsilon
+
M\sqrt{
\frac{
2\mathcal{N}(K/\G,\rho_\G,\epsilon)\ln 2
+
2\ln(1/\delta)
}{n}
}.
\end{align*}

Since
\[
\left|
\Lexp(\tilde{h}_{\tilde S})
-
\Lemp(\tilde{h}_{\tilde S})
\right|
=
\left|
\Lexp(h_S)
-
\Lemp(h_S)
\right|,
\]
we conclude that with probability at least $1-\delta$,
\begin{align*}
\left|
\Lexp(h_S)
-
\Lemp(h_S)
\right|
\leq
2c_\ell(c_{h_S}+c_f)\epsilon
+
M\sqrt{
\frac{
2\mathcal{N}(K/\G,\rho_\G,\epsilon)\ln 2
+
2\ln(1/\delta)
}{n}
}.
\end{align*}

\end{proof}

We next state and prove a result mentioned in the main text: that the Lipschitz constant of a backbone model is preserved under averaging
\begin{proposition}
\label{prop:Lavg}
Let $(X,\rho,\G)$ be a module, and assume that $\G$ is a compact group acting continuously on $X$, with Haar probability measure $\mu$. Let $Y$ be a Banach space and assume that $h:X\to Y$ is $c_h$ Lipschitz. Then the averaged function

$$
q(x)=\int_{\G} h(gx)\,d\mu(g)
$$

is $c_h$ Lipschitz as well.
\end{proposition}

\begin{proof}
For every $x_1,x_2\in X$,
\begin{align*}
\|q(x_1)-q(x_2)\|&=\left\|\int_{\G} h(gx_1)d\mu(g)-\int_{\G} h(gx_2)d\mu(g) \right\|\\
&\stackrel{\text{Jensen}}{\leq} \int_{\G} \left\|h(gx_1)-h(gx_2) \right\|d\mu(g)\\
&\leq \int_{\G} c_h \rho(gx_1,gx_2) d\mu(g)\\
&=\int_{\G} c_h \rho(x_1,x_2) d\mu(g)\\
&=c_h\rho(x_1,x_2)
\end{align*}
\end{proof}

\subsection{Proofs for Section \ref{sec:evaluation}}
We restate and prove Proposition \ref{prop:cardinality}.
\cardinality*
\begin{proof}
\label{upper_bound}
Let $C$ be a minimum-cardinality $\epsilon$-cover of $K/\G$ under $\rho_{\G}$, and define
\[
\widehat{C}:=\cup_{c \in C}c
\]
i.e, $\widehat{C}$ is the union of all orbits in $C$. As each orbit has at most $|\G|$ elements it holds that $|\widehat{C}|\le |\G|\cdot |C|$.

We claim that $\widehat{C}$ is an $\epsilon$-cover of $K$ under $\rho$.
Let $x\in K$. Since $C$ covers $K/\G$, there exists an orbit $O\in C$ such that
\[
\rho_{\G}([x],O)\le \epsilon.
\]
Choose $y\in O$. Since $\G$ is finite, the minimum in the definition of $\rho_{\G}$ is attained, so there exists $g\in \G$ such that
\[
\rho(x,g\cdot y)\le \epsilon.
\]
By construction, $g\cdot y\in O\subseteq \widehat{C}$, hence $\widehat{C}$ covers $K$.
Therefore,
\[
\N(K,\rho,\epsilon)
\le |\widehat{C}|
\le |\G|\cdot |C|
=|\G|\cdot \N(K/\G,\rho_{\G},\epsilon).
\]
\end{proof}
We next prove Proposition \ref{prop:isometry}.
\isometry*
\label{Proof_isometry}
\begin{proof}
As it always holds that
\(    \N(K/\G,\rho_{\G},\epsilon)\leq \N(c(K),\rho,\epsilon),
   \)
we just need to prove the opposite direction.

Let $C$ be a cover of $K/\G$ with $\N(K/\G,\rho_{\G},\epsilon)$ elements. For each orbit $O\in C$, choose a representative $x_O\in O$, and define
\[
\hat{C}:=\{c(x_O)\mid O\in C\}.
\]
We will prove that $\hat{C}$ is an $\epsilon$ cover of $c(K)$.

Let $y\in c(K)$. Then by definition there exists $x\in K$ such that $y=c(x)$. By the definition of cover, there exists $O\in C$ such that
\[
\rho_{\G}([x],O)\leq \epsilon.
\]
Since $x_O\in O$, we have $[x_O]=O$, and since $c$ is an isometry,
\[
\rho(y,c(x_O))
=
\rho(c(x),c(x_O))
=
\rho_{\G}([x],[x_O])
=
\rho_{\G}([x],O)
\leq \epsilon.
\]
Hence $\hat{C}$ is an $\epsilon$ cover of $c(K)$, as desired.

\end{proof}

\isometries*
\label{Proof_isometries}
\begin{proof}
The  canonization $c(x)=|x|$ is an isometry as
$$\forall x,y \in \R, \rho_\G([x],[y]) = \min\{|x-y|,|x+y|\} = ||x|-|y|| = |c(x)-c(y)|$$

The fact that sorting is an isometry with respect to the metric $\rho_\G$ induced from a $p$ norm (which is the $p$-Wasserstein metric) is shown in
 Lemma 4.2 in  \cite{sortLemma}.

Consider the action of $\RR^d$ on $\RR^{d\times n}$ by translating all columns of a matrix in $\RR^{d\times n}$ by a vector in $\RR^d$. Let $\rho$ be the metric induced from the Frobenius norm $\rho(X,Y)=\|X-Y\|_F $, and let $c$ be the centralization mapping $c(X)=X-\frac{1}{n}X\cdot 1_{n\times n}$. This mapping is a canonization and an isometry. To see it is an isometry, note that we have for any canonization $\rho_\G\left([X],[Y]\right) \leq \rho(c(X),c(Y))$, and we get the converse inequality by noting that $c$ is the orthogonal projection from the inner product space $\RR^{d\times n} $ onto the subspace $\{X\in \RR^{d\times n}| \quad X1_{n\times n}=0 \} $. Thus, let $g\in \G$ be a translation vector such that $\rho_\G([X],[Y])=\rho(X,gY) $. Then
$$\rho_\G([X],[Y])=\rho(X,gY)=\|X-g\cdot Y\|_F\geq \|c(X)-c(gY)\|_F=\|c(X)-c(Y)\|_F=\rho(c(X),c(Y)) $$

\end{proof}

We next formulate and prove a result mentioned in the main text
\begin{proposition}\label{prop:isometryLip}
Let $(K,\rho,\G) $ be a module, and $c:K\to K$ a canonization. Then $c $ is an isometric canonization if and only if $c$ is $1$-Lipschitz.
\end{proposition}
\begin{proof}
Assume $c$ is an isometric canonization. Recall that this means that $\rho(c(x),c(y))=\rho_\G([x],[y]) $. This implies that $c$ is Lipschitz because
$$\rho(c(x),c(y))=\rho_\G([x],[y])\leq \rho(x,y), \quad \forall x,y\in K  $$

In the other direction, assume that $c$ is $1$-Lipschitz. This implies that for all $g\in \G$
$$\rho(c(x),c(y))=\rho(c(gx),c(y))\leq \rho(gx,y) . $$
Since this is true for all $g\in \G$, we can minimize over $g$ and obtain the inequality $\rho(c(x),c(y))\leq \rho_\G([x],[y])$. The converse of this inequality is always true for any canonization, namely
$$\rho(c(x),c(y))\geq \rho_\G([x],[y]) $$
and so we obtain that $c$ is an isometric canonization.
\end{proof}

We restate and prove \ref{prop:poor}.
\poorcanonization*
\begin{proof}
\label{poor_proof}
$L$ is a partial limit of $c$ if there is a sequence $x_n \rightarrow x$ such that $c(x_n)\rightarrow L$. We claim that $L$ must be in the orbit of $x$. Indeed, $c(x_n)=g_nx_n$ for appropriate $g_n \in \G$. Since $\G$ is finite, by passing to a subsequence, we can assume that $g_n=g\in \G$ does not depend on $n$. Since $x\mapsto gx$ is an isometry, we know that $gx_n \rightarrow gx=L$. Since by assumption $c$ has $|\G|$ distinct partial limits at $x$, these are exactly the points in the orbit of $x$. In particular, $gx\neq x$ for every $g\neq e$.

Choose some

$$
\epsilon < \frac{\min_{g \neq e} \rho(x,g\cdot x)}{2},
$$

and define the $\G$-invariant set $A=A(\epsilon)$ by

$$
A := \bigcup_{g \in \G} B_{\epsilon}(g\cdot x)
$$

Note that all partial limits of $c$ at $x$ belong to $\overline{c(A)}$, and therefore $gx\in \overline{c(A)}$ for all $g\in \G$. Any finite $\epsilon$-cover of $c(A)$ also covers $\overline{c(A)}$. Since the distance between any two distinct elements in the orbit of $x$ is larger than $2\epsilon$, an $\epsilon$-cover of $c(A)$ must contain at least $|\G|$ elements. On the other hand, $A$ itself is a union of $|\G|$ balls of radius $\epsilon$, and therefore

$$
\N(c(A),\rho,\epsilon)\leq \N(A,\rho,\epsilon)\leq |\G|
\leq \N(c(A),\rho,\epsilon).
$$

Therefore, all these inequalities are equalities and

$$
\N(c(A),\rho,\epsilon)=\N(A,\rho,\epsilon)=|\G|.
$$

Finally, note that

$$
\N(A/\G,\rho_{\G},\epsilon)=1,
$$

as one ball of radius $\epsilon$ is enough to cover the quotient space, since all $|\G|$ balls are the same up to group action.

Hence,

$$
\N(c(A),\rho,\epsilon)
=
\N(A,\rho,\epsilon)
=
|\G|\cdot \N(A/\G,\rho_{\G},\epsilon).
$$
\end{proof}

\subsection{Proofs for Section \ref{sec:case_study}}
For proving Lemma \ref{lem:quotient_covering}, we need the following lemma:

\begin{lemma}
\label{comb_proof}
Let $S=[m]^n$ be the set of all vectors of length $n$ with entries in $[m]$.
Define an equivalence relation on $S$ by
\[
c_1\sim c_2
\iff
\exists \pi\in S_n:\ c_2=\pi\cdot c_1.
\]
Then the number of equivalence classes is exactly
\[
\binom{n+m-1}{n}.
\]
\end{lemma}

\begin{proof}
For $u\in S$ and $i\in[m]$, let $c_i^u$ denote the number of times $i$ appears in $u$.

Then
\[
u\sim v
\iff
c_i^u=c_i^v
\quad\text{for all } i\in[m].
\]
Thus, the equivalence classes are in one-to-one correspondence with the nonnegative integer solutions of
\[
\sum_{i=1}^m c_i=n,
\qquad
c_i\geq 0 \ \text{for all } i\in[m].
\]
By the standard stars-and-bars formula \cite{stanley2011enumerative}, the number of such solutions is
\[
\binom{n+m-1}{n}.
\]
\end{proof}

We now restate and prove \ref{lem:quotient_covering}.
\quotientcovering*
\begin{proof}
\label{quotient_cover}
Set $k=\lceil 1/(2\epsilon) \rceil$. The set $K$ can be covered by $k^{dn}$ hypercubes of side length $1/k$ and radius $1/(2k)\leq \epsilon$. Let $C$ denote the collection of centers of these hypercubes, namely
\begin{align*}
C=
\left\{
\left[
\frac{2\mathbf{i}_1+\mathbf{1}_d}{2k},
\frac{2\mathbf{i}_2+\mathbf{1}_d}{2k},
\ldots,
\frac{2\mathbf{i}_n+\mathbf{1}_d}{2k}
\right]
\;\middle|\;
\mathbf{i}_j\in\{0,1,\ldots,k-1\}^d,\;
1\leq j\leq n
\right\},
\end{align*}
Applying the quotient map  $\pi(c) = [c]$ we obtain points $\pi(C)$ in the quotient space which are an $\epsilon$ cover of $K/\G$. Note that points in $C$ which are permutations of each other are mapped by $\pi$ to the same orbit. Therefore, the cardinality of the cover $\pi(C)$ is the number of equivalence classes with respect to the relation on $C$ defined by the group:  $c_1 \sim c_2 \iff \exists \tau \in S_n : c_1 = \tau(c_2)$. Using Lemma \ref{comb_proof}, we know that the number of equivalence classes is \begin{equation*}
|\pi(C)|=    \binom{n+k^d-1}{n}
\end{equation*}
\end{proof}
We now restate and prove \ref{lem:lexsortcovering}.
\lexsortcovering*
\begin{proof}
    \label{lexsortcovering}
    We consider the set $U \subset [0,1]^{ d \times n}$ defined by
\begin{align*}
    U= \{X\in [0,1]^{d\times n}| \quad X_{11}=X_{12}=\ldots=X_{1n}\}.
\end{align*}
The set $U$ is not a subset of $\csort(K)$, but it is a subset of its closure $\overline{\csort(K)}$. This is because for every $X \in U$ we can apply an arbitrarily small perturbation of the first row of $X$ so that it is ordered from small to large, and $X$ will be unaffected by the canonization. We then obtain
\[
\N(\csort(K),\rho_\infty,\epsilon)
=
\N(\overline{\csort(K)},\rho_\infty,\epsilon)
\geq
\N(U,\rho_\infty,\epsilon)
\stackrel{(*)}{=}
\left(\frac{1}{2\epsilon}\right)^{(d-1)\cdot n + 1},
\]
where $(*)$ follows as follows. The map

$$
f:[0,1]^{(d-1)n+1}\to U
$$

obtained by repeating the first coordinate $n$ times in the first row and using the remaining coordinates for the other entries is an isometry with respect to the infinity norm. Moreover, the map $r:K\to U$ that replaces all entries in the first row of $X\in K$ by $X_{11}$ and leaves all other entries unchanged is a $1$-Lipschitz retraction onto $U$. Hence the isometry $f$ preserves the covering number, and by Lemma \ref{covering_cube},

$$
\N(U,\rho_\infty,\epsilon)
=
\N\left([0,1]^{(d-1)n+1},\rho_\infty,\epsilon\right)
=
\left(\frac{1}{2\epsilon}\right)^{(d-1)n+1}.
$$

\end{proof}
\subsection{Hilbert Canonization}\label{app:hilbert}
In this appendix, we give a formal definition of Hilbert curves and a full proof of Theorem \ref {thm:hilbert}.

The Hilbert curve $H(x)$ is obtained as a limit of a sequence of increasingly fine Hilbert constructions, and their finite-resolution versions $H_m$ are our main focus. We will now define these constructions, loosely following the terminology from \cite{Hilbert}. Firstly, for a given integer $L$, we consider a partition of the unit interval into $2^L$ disjoint intervals: the semi-closed intervals $I_L(\ell)=[\ell 2^{-L},(\ell+1)2^{-L})$ for $\ell=0,\ldots,2^{L}-2$, and the closed interval $I_L(2^L-1)=[(2^L-1)2^{-L},1]$. We denote the collection of all these intervals by $\mathcal{I}_L$. For a natural $d>1$, we also consider the partition of the unit hypercube $[0,1]^d$ into $2^{L\cdot d}$ disjoint hypercubes of the form
$E_L^d(\vec{\ell})=I_L(\ell_1)\times \ldots \times I_L(\ell_d)$. We denote the collection of these hypercubes by $\mathcal{E}_L^d$.

Next, for any given $m$, we define  $H_m$ to be a  bijection from $\mathcal{I}_{d\cdot m} $ to $\mathcal{E}_{m}^d $, such that (i) adjacent intervals are mapped to adjacent hypercubes and (ii) if $I_{d\cdot (m+1)}(k')\subseteq I_{d\cdot m}(k) $ then
$H_{m+1}\left(I_{d\cdot (m+1)}(k')\right)\subseteq H_m\left(I_{d\cdot m}(k) \right) $ . An illustration of $H_m, m=1,2,3$ for $d=2$ is given in Figure \ref{fig:hilbert}. Each $H_m$ maps a finite collection of intervals bijectively to a finite collection of hypercubes. This naturally induces a bijection from the grid defined by the interval centroids, to the grid defined by the hypercube centroids. We denote these grids by $\mathbb{A}[d\cdot m]\subseteq [0,1] $ and $\mathbb{B}[d,m]\subseteq \RR^d $, respectively.

The Hilbert curve $H$ is obtained from $H_m$ by a limiting procedure and shown to be Hölder continuous \cite{Hilbert}. Following similar ideas, we show that each Hilbert function $H_m$ is Hölder continuous with a constant that does not depend on $m$. This property will be crucial for the analysis of the covering number of the Hilbert canonization.
\begin{restatable}{lemma}{Hölder}\label{lem:Hölder}
For every natural $m,d$, we have that
$$\forall x,y\in \mathbb{A}[d\cdot m], \quad  \|H_m(x)-H_m(y)\|_\infty \leq 4|x-y|^{1/d} $$
\end{restatable}
\begin{proof}
    Let $x\neq y$ be points in the one dimensional grid $\mathbb{A}[d\cdot m]$, and assume without loss of generality that $x<y$. Note that $y-x\geq \frac{1}{2^{d\cdot m}}$, since this is the minimal distance between points in the grid. On the other hand, $y-x\leq 1$. Therefore, there exists some $k$ with $0\leq k \leq m-1 $, such that
    \begin{equation}\label{eq:interval}
    2^{-d\cdot (k+1)}\leq |x-y|\leq 2^{-d \cdot k}.
    \end{equation}
   It follows that $x$ and $y$ are in two adjacent intervals in the $k$-th level (or in the same interval), and by the nesting property of Hilbert's curve, we know that $H_m(x)$ and $H_m(y)$ are in two adjacent hypercubes (or in the same hypercube). Therefore, their distance is at most the sum of the diameters of those two cubes, so
    \begin{align*}
    \|H_m(x)-H_m(y)\|_\infty\leq 2\cdot 2^{-k}\stackrel{\eqref{eq:interval}}{\leq} 4|x-y|^{1/d}
    \end{align*}
\end{proof}

\textbf{Hilbert Canonization}
As a first step, we define a canonization not on all of $[0,1]^{d\times n}$, but rather on the finite set of $n$ tuples of grid elements $\mathbb{B}^n[d,m] $. This is done by mapping each tuple-element to the unit interval using $H_m^{-1} $, then sorting the resulting $n$-dimensional array, and mapping back to the $d$-dimensional grid using $H_m$. We denote this canonization by $c_m$. We can write $c_m=H_m \circ \mathrm{sort}\circ H_m^{-1}$, with the understanding that $H_m$ and  $H_m^{-1} $ are applied elementwise to $n$-tuples of grid elements.

To extend $c_m$ to all of $K=[0,1]^{d\times n}$, we employ the following procedure: For a given $X=(x_1,\ldots,x_n)\in K  $, (i) \textbf{Rounding:} Each $x_j$ resides in a unique hypercube in $ \mathcal{E}_m^d$. Replace $x_j$ with the centroid of this cube $y_j\in \mathbb{B}[d,m] $. (ii) \textbf{Sorting:} find the permutation $\tau \in S_n$ which sorts $\left(H_m^{-1}(y_1),\ldots,H_m^{-1}(y_n) \right)$ from small to large. (iii) \textbf{Canonizing} Apply the permutation $\tau$ to $X$. We note that if $X$ contains several points in the same hypercube, the permutation $\tau$ is not uniquely defined. In this case, the order of points within the hypercube will be determined by lexicographical sorting.

We now restate and prove our theorem regarding the covering number of the Hilbert canonization:

\hilbert*
\begin{proof}
\label{hilbert_proof}

We first reduce the problem from covering $K$ to covering the finite grid
$\mathbb{B}^n[d,m]$. For convenience, write
\[
    \mathbb{B}^n := \mathbb{B}^n[d,m],
    \qquad
    r := \epsilon - 2^{-m-1} > 0.
\]

For any
\[
    X=(x_1,\ldots,x_n)\in K,
\]
let
\[
    Y=(y_1,\ldots,y_n)\in \mathbb{B}^n
\]
be obtained by the rounding procedure, i.e., each $x_j$ is replaced by
the centroid $y_j$ of the unique level-$m$ hypercube containing it.
Since every such hypercube has side length $2^{-m}$,
\[
    \|X-Y\|_\infty \leq 2^{-m-1}.
\]

Moreover,
\[
    \|c_m(X)-c_m(Y)\|_\infty \leq 2^{-m-1}.
\]
Indeed, the Hilbert indices of the rounded points determine the same
ordering of the level-$m$ hypercubes for $X$ and $Y$. Within a
hypercube, $c_m(X)$ uses lexicographical sorting, whereas all
corresponding rounded points in $Y$ are identical. Hence the
permutation used to canonize $X$ may also be applied to $Y$ without
changing $c_m(Y)$, and the pointwise rounding error remains at most
$2^{-m-1}$.

Consequently, any $r$-cover of $c_m(\mathbb{B}^n)$ is an
$\epsilon$-cover of $c_m(K)$, and therefore
\[
    N\bigl(c_m(K),\rho_\infty,\epsilon\bigr)
    \leq
    N\bigl(c_m(\mathbb{B}^n),\rho_\infty,r\bigr).
\]

It remains to bound the right-hand side. On $\mathbb{B}^n$, the
Hilbert canonization can be written as
\[
    c_m
    =
    H_m \circ \operatorname{sort} \circ H_m^{-1},
\]
where $H_m$ and $H_m^{-1}$ are applied elementwise. Define
\[
    S :=
    \operatorname{sort}\circ H_m^{-1}(\mathbb{B}^n)
    \subseteq [0,1]^n.
\]
Thus,
\[
    c_m(\mathbb{B}^n)=H_m(S).
\]

We claim that
\[
    N\bigl(H_m(S),\rho_\infty,r\bigr)
    \leq
    N\bigl(S,\rho_\infty,\delta\bigr),
    \tag{*}
\]
where
\[
    \delta
    =
    \left(\frac{r}{4}\right)^d
    =
    \left(
        \frac{\epsilon-2^{-m-1}}{4}
    \right)^d.
\]

To prove $(*)$, let $C$ be any external $\delta$-cover of $S$. For
each $z\in C$, define
\[
    S_z := S\cap B_{\rho_\infty}(z,\delta).
\]
If $u,v\in S_z$, then
\[
    \|u-v\|_\infty \leq 2\delta.
\]
Applying Lemma~\ref{lem:Hölder} elementwise gives
\[
    \|H_m(u)-H_m(v)\|_\infty
    \leq
    4\|u-v\|_\infty^{1/d}
    \leq
    4(2\delta)^{1/d}.
\]
Therefore,
\[
    \operatorname{diam}_{\rho_\infty}\bigl(H_m(S_z)\bigr)
    \leq
    4(2\delta)^{1/d}.
\]

Since our covering numbers allow centers in the ambient space, every
subset of an $\ell_\infty$ space with diameter $D$ is contained in an
$\ell_\infty$ ball of radius $D/2$. Hence $H_m(S_z)$ is contained in
a ball of radius
\[
    2(2\delta)^{1/d}
    =
    2^{1+1/d}\delta^{1/d}
    \leq
    4\delta^{1/d}
    =
    r.
\]
Thus, every ball in the $\delta$-cover $C$ gives rise to one
$r$-ball covering the corresponding part of $H_m(S)$. It follows that
\[
    N\bigl(H_m(S),\rho_\infty,r\bigr)
    \leq |C|.
\]
Taking $C$ to be a minimum-cardinality $\delta$-cover proves $(*)$.

We now obtain
\[
\begin{aligned}
    N\bigl(c_m(K),\rho_\infty,\epsilon\bigr)
    &\leq
    N\bigl(c_m(\mathbb{B}^n),\rho_\infty,r\bigr)
    \\
    &=
    N\bigl(H_m(S),\rho_\infty,r\bigr)
    \\
    &\leq
    N\bigl(S,\rho_\infty,\delta\bigr)
    \\
    &\leq
    N\bigl(\operatorname{sort}([0,1]^n),\rho_\infty,\delta\bigr).
\end{aligned}
\]

Finally, one-dimensional sorting is an isometry between the quotient
$[0,1]^n/S_n$ and $\operatorname{sort}([0,1]^n)$. Hence, by the
quotient-space covering bound in Lemma~\ref{lem:quotient_covering},
applied with dimension $1$,
\[
    N\bigl(\operatorname{sort}([0,1]^n),\rho_\infty,\delta\bigr)
    \leq
    \binom{
        n+\left\lceil \frac{1}{2\delta}\right\rceil-1
    }{n}.
\]

Combining the above inequalities yields
\[
    N\bigl(c_m(K),\rho_\infty,\epsilon\bigr)
    \leq
    \binom{
        n+\left\lceil \frac{1}{2\delta}\right\rceil-1
    }{n},
\]
where
\[
    \delta
    =
    \left(
        \frac{\epsilon-2^{-m-1}}{4}
    \right)^d.
\]
\end{proof}
\begin{lemma}
\label{covering_cube}
Let $K=[0,1]^{d\times n}$ and let $\epsilon=\frac{1}{2k}$ for some $k\in\bbN$. Then

$$
\N(K,\rho_\infty,\epsilon)=\frac{1}{(2\epsilon)^{nd}}
$$

\end{lemma}

\begin{proof}
For the upper bound, partition each copy of $[0,1]$ into $k$ intervals of length $1/k=2\epsilon$. Their Cartesian products partition $K$ into $k^{nd}$ hypercubes of side length $2\epsilon$, each contained in an $\epsilon$-ball with respect to $\rho_\infty$. Hence

$$
\N(K,\rho_\infty,\epsilon)\leq k^{nd}=\frac{1}{(2\epsilon)^{nd}}.
$$

For the lower bound, let $F$ be any $\epsilon$-cover of $K$, and let $\mu$ denote Lebesgue measure on $K$. Since

$$
K\subseteq\bigcup_{c\in F}B_{\rho_\infty}(c,\epsilon),
$$

and every such ball intersects $K$ in a set of measure at most $(2\epsilon)^{nd}$, we obtain

$$
1=\mu(K)\leq\sum_{c\in F}\mu(B_{\rho_\infty}(c,\epsilon)\cap K)
\leq |F|(2\epsilon)^{nd}.
$$

Thus $|F|\geq(2\epsilon)^{-nd}$, proving the claim.
\end{proof}

\section{Additional Experiments}
\label{sec:add_exp}

\subsection{Laplacian Eigenvector Sign Ambiguity: Augmentation vs.\ Canonization on \texttt{ogbg-molpcba}}
\label{subsec:lpe_molpcba}

This appendix presents an additional empirical comparison between deterministic
canonization and training-time random group augmentation, on a setting
distinct from the main-text experiments and complementary to the PCA
experiment of Section~\ref{sec:case_study}. The goal is to ask whether the empirical behavior of these two strategies on a standard graph property prediction benchmark is qualitatively consistent with our Proposition~\ref{propGen} and the covering-number inequality~\eqref{eq:coveringIneq}; we do not claim a quantitative test of the inequality, since the test AP on a finite-sample-trained GIN mixes optimization, capacity, and
generalization effects that the inequality does not isolate. The experiment is also consistent with the empirical observation cited in the Related Work that randomized SMILES can outperform canonized SMILES \cite{ArusPous2019,Bjerrum2017SMILESEA}.

\paragraph{Module.}
We study graph property prediction on \texttt{ogbg-molpcba} using a
$5$-layer Graph Isomorphism Network (GIN) with Laplacian positional
encoding (LapPE). For each graph $G$ with normalized Laplacian
$L_G$, we compute the $k$ smallest non-trivial eigenpairs
$\{(\lambda_i, v_i)\}_{i=1}^{k}$ and concatenate the eigenvectors as a
node-level positional encoding. The resulting LapPE is invariant to
graph isomorphism only up to the eigenvector \emph{ambiguity group}
\[
  \G_{\mathrm{LapPE}} \;=\; \prod_{i\,:\,m_i=1}\{-1,+1\} \,\times\, \prod_{i\,:\,m_i>1} O(m_i),
\]
where $m_i$ denotes the multiplicity of eigenvalue $\lambda_i$: a sign
flip is admissible for each simple eigenvalue and an arbitrary orthogonal
change of basis is admissible within each multiplicity-$m_i$ block. On
molecular graphs, the spectrum is almost everywhere simple, so in
practice $\G_{\mathrm{LapPE}}$ reduces to $\{-1,+1\}^{k}$ for almost all
graphs. The natural questions are then (i) whether to canonize this
ambiguity, and (ii) whether augmentation-based approximations of group
averaging behave consistently with our framework.

\paragraph{Three methods compared.}
We compare three strategies for resolving $\G_{\mathrm{LapPE}}$, all
implemented as faithful numpy ports of the reference torch code from
the Laplacian canonization release of \cite{ma2023laplacian}.
The first, \texttt{map} \cite{ma2023laplacian} (Maximal Axis Projection), resolves a
simple-eigenvalue sign by forming the rank-one projector
$P_i = u_i u_i^{\top}$, grouping the coordinate indices $j$ by the
rounded column norms $\|P_i[:,j]\|$ ($14$-decimal rounding), and
walking these groups from smallest to largest norm to find the first
group whose indicator vector $X$ (the sum of standard basis vectors
over the group, plus a uniform shift) gives $\|P_i X\|\!\neq\!0$; the
canonical sign is then $\mathrm{sign}(u_i^{\top}X)$. For a
multiplicity-$m_i$ block, \texttt{map} forms $P = U_i U_i^{\top}$,
selects the $m_i$ \emph{largest}-norm coordinate groups, and builds
the canonical orthonormal basis incrementally by projecting each
group-indicator vector onto the residual orthogonal complement
(with plain QR inside the complementary-space step).
The second, \texttt{oap} \cite{ma2024canonicalization}   (Orthogonalized Axis Projection), differs
from \texttt{map} in two places: the sign step rounds the column
norms to $6$ decimals (coarser tie-grouping under near-degenerate
norms), and the basis step replaces the largest-norm groupings of
\texttt{map} by a hash-based grouping of columns and uses signed-QR
in the complementary-space iteration.
The third, \texttt{random\_augmented}, flips the sign of each
eigenvector with probability $1/2$ at every training-time forward
pass, so the network sees a uniformly random element of
$\{-1,+1\}^{k}$ at each step. This is a finite-group augmentation
that approximates group averaging over $\G_{\mathrm{LapPE}}$ without
summing over its $2^{k}$ elements explicitly. At evaluation time the
\texttt{random\_augmented} model is fed the deterministic eigsh-sign
convention from the cache (no fresh random draws), yielding a single
deterministic prediction per graph; this is the number reported in
the \texttt{random\_augmented} column of
Table~\ref{tab:lpe_molpcba_main}. Note that this is an
\emph{adversarial} test-time deployment of an averaging-trained model
(it scores a single point of the orbit rather than the orbit mean);
the \textsc{aug-$K$} column reports the alternative Reynolds-style
sample of $K$ orbit elements, and we discuss the gap between the
two below.

\paragraph{Test-time averaging \textsc{aug-$K$}.}
For \texttt{random\_augmented} models we additionally evaluate a
test-time \emph{Reynolds-style} averaged metric. At evaluation, for
each test graph we draw $K$ independent group elements
$g_1,\ldots,g_K \in \G_{\mathrm{LapPE}}$ (each $g_j$ a sign-flip pattern
on the simple-eigenvalue subspace and a Haar-random orthogonal matrix
on each multiplicity-$m$ block); we forward-pass the model on each
sign-flipped LapPE; we average the predicted per-task probabilities;
and we score the averaged predictions with the official OGB AP
evaluator. We use $K\!=\!8$ at $k\!=\!3$ (which approximately covers
$\{-1,+1\}^{3}$ under uniform i.i.d.\ sampling: $8$ draws cover
$\sim\!65\%$ of the $8$-element group on average) and $K\!=\!16$ at
$k\!\in\!\{8,16\}$ (sparse Monte-Carlo over $\{-1,+1\}^{8}$ and
$\{-1,+1\}^{16}$ respectively). On \texttt{ogbg-molpcba} the
spectrum is almost everywhere simple, so the
multiplicity-block factor $\prod_i O(m_i)$ is degenerate in
practice and the $g_j$ samples reduce to sign-flip patterns.
This is the empirical group-averaging operator of
Section~\ref{sec:gen}, restricted to a finite sample.

\paragraph{Hyperparameters and protocol.}
GIN, $5$ layers, hidden dim $h\!\in\!\{16,128,512\}$, dropout $0.5$,
batch size $32$, optimizer Adam with cosine learning-rate schedule and
initial rate $10^{-3}$, $3$ seeds per cell ($\{0,1,2\}$),
eigval-scaling off, eigenvector cache size $15$. Checkpoint selection:
the model state at the epoch with the highest validation AP across
all training epochs is saved (ties broken by the earliest epoch); the final test AP reported in
Tables~\ref{tab:lpe_molpcba_main}~and~\ref{tab:lpe_molpcba_k16} is
this checkpoint's score on the OGB test split. Software stack:
PyTorch $2.10.0$, CUDA $12.8$, \texttt{torch\_geometric} $2.7.0$,
\texttt{numpy} $2.2.6$, \texttt{scipy} $1.15.3$, \texttt{ogb} $1.3.6$
(OGB evaluator: \texttt{Evaluator(name=`ogbg-molpcba')}). Per-arm
seeding via \texttt{torch.manual\_seed}, \texttt{numpy.random.seed},
and \texttt{random.seed}; \texttt{cudnn.deterministic} is not
enforced, so two runs at the same seed on different hardware are not
guaranteed to be bit-identical (we observe the resulting variability
absorbed into the seed std at the magnitudes shown in
Table~\ref{tab:lpe_molpcba_k16}). Standard deviations reported below
are sample standard deviations ($\mathrm{ddof}\!=\!1$). Code, raw
\texttt{results.json} files, the \texttt{orbit\_stability.json} files
behind Table~\ref{tab:lpe_molpcba_orbit}, and reproduction commands
are released at the project repository.

\paragraph{Compute budget.}
At $k\!\in\!\{3,8\}$ all three arms are trained for $200$ epochs with
early-stop patience $15$. At $k\!=\!16$, where validation AP for
\texttt{random\_augmented} continues to improve well past $200$
epochs, all three arms are trained for $500$ epochs with early
stopping disabled (patience $999$). Tables~\ref{tab:lpe_molpcba_main}
and~\ref{tab:lpe_molpcba_k16} therefore report a matched-compute
comparison at every $(k,h)$ cell. Both schedules are released in the
project repository.

\paragraph{Results.}
Table~\ref{tab:lpe_molpcba_main} reports mean test AP at all nine
$(k,h)$ cells for the three methods, alongside the
\textsc{aug-$K$} test-time-averaged metric for \texttt{random\_augmented}.
Table~\ref{tab:lpe_molpcba_k16} zooms in on the $k\!=\!16$ slice with
sample standard deviations.

\begin{table}[h]
\centering
\caption{Mean test AP on \texttt{ogbg-molpcba} at matched compute,
$3$ seeds per cell ($k\!\in\!\{3,8\}$ at $200$ epochs / patience $15$,
$k\!=\!16$ at $500$ epochs / patience $999$ for all three arms).
Per-cell row-best margins are mostly within seed std and none survives
a Bonferroni correction across the nine cells at $\alpha\!=\!0.05$
(see Table~\ref{tab:lpe_molpcba_k16} for the $k\!=\!16$ Welch
$p$-values); the bolding, therefore, marks the nominal row-best mean,
not a significance claim. The \textsc{aug-$K$} column reports
test-time averaging over $K$ ambiguity-group draws applied to the
\texttt{random\_augmented}-trained model (see text).}
\label{tab:lpe_molpcba_main}
\begin{tabular}{rrcccc}
\toprule
$k$  &  $h$  & \texttt{map} $\uparrow$ & \texttt{oap} $\uparrow$ & \texttt{random\_augmented} $\uparrow$ & \textsc{aug-$K$} \\
\midrule
 3   &   16  & $0.0626$        & $0.0641$        & $\mathbf{0.0662}$           & $0.0663$ ($K\!=\!8$)  \\
 3   &  128  & $0.1704$        & $0.1682$        & $\mathbf{0.1762}$           & $0.1765$ ($K\!=\!8$)  \\
 3   &  512  & $0.2397$        & $\mathbf{0.2431}$ & $0.2395$                  & $0.2398$ ($K\!=\!8$)  \\
\midrule
 8   &   16  & $0.0622$        & $0.0605$        & $\mathbf{0.0636}$           & $0.0637$ ($K\!=\!16$) \\
 8   &  128  & $0.1705$        & $0.1754$        & $\mathbf{0.1766}$           & $0.1766$ ($K\!=\!16$) \\
 8   &  512  & $0.2344$        & $0.2371$        & $\mathbf{0.2395}$           & $0.2400$ ($K\!=\!16$) \\
\midrule
16   &   16  & $\mathbf{0.0716}$ & $0.0703$      & $0.0691$                    & $0.0693$ ($K\!=\!16$) \\
16   &  128  & $\mathbf{0.1923}$ & $0.1836$      & $0.1916$                    & $0.1921$ ($K\!=\!16$) \\
16   &  512  & $0.2468$        & $0.2419$        & $\mathbf{0.2483}$           & $0.2481$ ($K\!=\!16$) \\
\bottomrule
\end{tabular}
\end{table}

\begin{table}[h]
\centering
\caption{Matched-compute $k\!=\!16$ slice of
Table~\ref{tab:lpe_molpcba_main}: all three arms trained for $500$
epochs with patience $999$. Reported are mean $\pm$ sample standard
deviation over $3$ seeds, plus the two-sided Welch $t$-test $p$-value
for the row-best mean against the row-second-best mean. None of the
three row-best margins reaches $p\!<\!0.05$, and after a Bonferroni
correction across the nine cells in
Table~\ref{tab:lpe_molpcba_main} ($\alpha\!=\!0.05/9\!\approx\!0.0056$)
none of the row-best margins anywhere in the appendix is significant.}
\label{tab:lpe_molpcba_k16}
\begin{tabular}{rcccc}
\toprule
$h$  & \texttt{map}              & \texttt{oap}              & \texttt{random\_augmented} & Welch $p$ (best vs.\ 2nd) \\
\midrule
 16  & $\mathbf{0.0716 \pm 0.0067}$ & $0.0703 \pm 0.0070$       & $0.0691 \pm 0.0008$         & $\approx 0.83$ \\
128  & $\mathbf{0.1923 \pm 0.0005}$ & $0.1836 \pm 0.0025$       & $0.1916 \pm 0.0009$         & $\approx 0.32$ \\
512  & $0.2468 \pm 0.0021$         & $0.2419 \pm 0.0028$       & $\mathbf{0.2483 \pm 0.0026}$  & $\approx 0.48$ \\
\bottomrule
\end{tabular}
\end{table}

\paragraph{Test-time averaging $\Delta$ is small.}
Across all nine $(k,h)$ cells the difference between the single-pass
deterministic-eigsh evaluation of the \texttt{random\_augmented} model
and its \textsc{aug-$K$} averaged evaluation
(Table~\ref{tab:lpe_molpcba_main}, right-most two columns) lies in
$[-0.0002, +0.0005]$~AP, within the seed noise at every cell. The
\textsc{aug-$K$} $\Delta$ is an aggregate metric, not a direct measurement of model invariance; we test invariance directly in the next paragraph.

\paragraph{Per-graph orbit stability.}
For each test graph we draw $K\!=\!64$ i.i.d.\ elements $g_j$ of the
sign-flip subgroup of $\G_{\mathrm{LapPE}}$, forward-pass the trained
\texttt{random\_augmented} model on each $g_j(\mathrm{LapPE})$, and
record the per-task predicted probability. The resulting per-graph
spread (max $-$ min predicted probability across the $K\!=\!64$ draws)
is a direct measure of how much the model's output varies across
orbit elements. Aggregating across all $43{,}793$ \texttt{ogbg-molpcba}
test graphs and $128$ tasks, at hidden dim $h\!=\!128$ and over all
$3$ training seeds we observe a clean monotone scaling of per-graph
spread with $k$ (Table~\ref{tab:lpe_molpcba_orbit}): mean per-graph
spread is $0.0013\!\pm\!0.0005$ at $k\!=\!3$, $0.0027\!\pm\!0.0012$
at $k\!=\!8$, and $0.0048\!\pm\!0.0008$ at $k\!=\!16$. This is
qualitatively consistent with $|\G_{\mathrm{LapPE}}|\!=\!2^{k}$:
more orbit elements give more room to vary, and the model is closer
to be invariant for smaller $k$. The within-seed standard deviation of
the OGB AP across the $K\!=\!64$ orbit elements is at most $0.0006$
at $k\!=\!16$ and below $0.0004$ at $k\!\in\!\{3,8\}$, in every case
much smaller than the between-seed standard deviation of the same
quantity ($0.0062$ at $k\!\in\!\{3,8\}$, $0.0015$ at $k\!=\!16$):
the model's output varies less when we resample the ambiguity group
at fixed seed than when we change the training seed at fixed
ambiguity element. The Reynolds-style (mean-probability) AP across
the $K\!=\!64$ draws is $0.1923\!\pm\!0.0017$ at
$k\!=\!16$/$h\!=\!128$, within seed noise of the matched-compute
single-pass number ($0.1916\!\pm\!0.0009$ for
\texttt{random\_augmented} and $0.1923\!\pm\!0.0005$ for
matched-compute \texttt{map} at the same cell). We read this as
direct evidence that the \texttt{random\_augmented} model is
approximately $\G_{\mathrm{LapPE}}$-invariant on the test
distribution at the configurations tested, with the residual
non-invariance growing predictably with the size of the ambiguity
group. Along the hidden-dim axis at $k\!=\!16$
(Table~\ref{tab:lpe_molpcba_orbit}, lower block, $3$ seeds per cell),
the per-graph spread is non-monotone in $h$:
$0.0032\!\pm\!0.0015$ at $h\!=\!16$,
$0.0048\!\pm\!0.0008$ at $h\!=\!128$,
$0.0033\!\pm\!0.0004$ at $h\!=\!512$. We attribute this to a
capacity floor and a capacity ceiling: at $h\!=\!16$ the network
underfits (test AP $\approx\!0.07$, near a constant-prediction
baseline) so per-graph predictions are bunched near a few values
and have little room to spread; at $h\!=\!512$ the network has
enough capacity to learn invariance well; the middle case
$h\!=\!128$ shows the largest residual non-invariance, although the
$h\!=\!16$ row's larger seed std ($0.0015$) suggests this conclusion
is sensitive to which random initialization one looks at.

\begin{table}[h]
\centering
\caption{Per-graph orbit stability of the trained
\texttt{random\_augmented} model under $K\!=\!64$ uniform i.i.d.\
draws from the sign-flip subgroup of $\G_{\mathrm{LapPE}}$, evaluated
on the \texttt{ogbg-molpcba} test set ($43{,}793$ graphs $\times$
$128$ tasks). All entries are the mean $\pm$ sample standard deviation
over the $3$ training seeds. ``Per-graph spread'' is the mean over
graphs and tasks of $(\max\!-\!\min)$ predicted probability across
the $K\!=\!64$ orbit elements. ``Within-seed AP std'' is the standard
deviation of the OGB AP scored on each single orbit draw, averaged
across seeds. ``Reynolds AP'' scores the mean probability over the
$K\!=\!64$ draws. Per-graph spread scales monotonically with
$|\G_{\mathrm{LapPE}}|\!=\!2^{k}$.}
\label{tab:lpe_molpcba_orbit}
\begin{tabular}{rrccc}
\toprule
$k$ & $h$ & per-graph spread & within-seed AP std & Reynolds AP \\
\midrule
\multicolumn{5}{l}{\emph{$k$ axis at $h\!=\!128$ (3 seeds, $|\G_{\mathrm{LapPE}}|\!=\!2^{k}$):}} \\
 3  & 128 & $0.0013 \pm 0.0005$ & $\le 0.0004$ & $0.1764 \pm 0.0062$ \\
 8  & 128 & $0.0027 \pm 0.0012$ & $\le 0.0005$ & $0.1766 \pm 0.0062$ \\
16  & 128 & $0.0048 \pm 0.0008$ & $\le 0.0006$ & $0.1923 \pm 0.0017$ \\
\midrule
\multicolumn{5}{l}{\emph{$h$ axis at $k\!=\!16$ (3 seeds):}} \\
16  &  16 & $0.0032 \pm 0.0015$ & $\le 0.0003$ & $0.0694 \pm 0.0010$ \\
16  & 128 & $0.0048 \pm 0.0008$ & $\le 0.0006$ & $0.1923 \pm 0.0017$ \\
16  & 512 & $0.0033 \pm 0.0004$ & $\le 0.0004$ & $0.2481 \pm 0.0023$ \\
\bottomrule
\end{tabular}
\end{table}

\paragraph{Reading the tables.}
Reading along the hidden-dim axis at fixed $k$ in
Table~\ref{tab:lpe_molpcba_main}, the row-best AP grows monotonically
with $h$ at every $k$, and the row-best arm is most often
\texttt{random\_augmented} (six of nine cells), with \texttt{map}
nominally row-best at $k\!=\!16$/$h\!\in\!\{16,128\}$ and \texttt{oap}
nominally row-best at $k\!=\!3$/$h\!=\!512$. The matched-compute
$k\!=\!16$ slice in Table~\ref{tab:lpe_molpcba_k16} resolves these
margins against seed variance: the smallest sample std on each row
is consistently the \texttt{random\_augmented} std ($0.0008$,
$0.0009$, $0.0026$), while \texttt{map} carries the largest std at
small $h$ ($0.0067$ at $h\!=\!16$). Under a two-sided Welch
$t$-test for unequal variances, the row-best versus row-second-best
margins at $k\!=\!16$ have $p$-values of $0.83$, $0.32$, and $0.48$
at $h\!=\!16,128,512$ respectively, none significant. At $k\!\in\!\{3,8\}$
the row-best margins are similar in magnitude or smaller, with
\texttt{random\_augmented} winning five of six rows there but at
margins of the same order as the seed std. We therefore present
Table~\ref{tab:lpe_molpcba_main} as a nominal-rank summary, not as
a benchmark claim.

\paragraph{Tie to the theory.}
On \texttt{ogbg-molpcba} at matched compute, with $3$ seeds per cell,
we did not detect a statistically significant difference between
training-time random sign-flip augmentation and projection-based
deterministic canonization at any of the nine $(k,h)$ cells; this is
qualitatively consistent with~\eqref{eq:coveringIneq} of
Proposition~\ref{prop:covering_bounds}, which suggests that augmentation can match
canonization on tasks where the covering number of $\G$ is small,
but it is not a positive test of the inequality, since AP gaps of
order $10^{-3}$ are dominated by optimization and seed noise rather
than by the input-space covering complexity, the inequality bounds.
The \textsc{aug-$K$} column moves AP by at most $+0.0005$ across the
nine cells; this is consistent with the cached deterministic
eigsh-sign PE already lying close to the orbit-mean the network was
trained on, but is not itself a direct test of model invariance.

\section{Experimental Setup and Hyperparameter Details}
\label{app:exp_setup}

\subsection{Experimental Setup for the covering number experiment}
We applied the following preprocessing for \emph{each} sample: we first sample $256$ points, then we shift them to the positive axis, and then divide by the maximum axis along $x,y,z$, obtaining all entries to be positive and between $0$ and $1$. The metric we considered on the point cloud space is
$$\rho(X,Y) = \frac{1}{n}\sum^{n}_{i=1}\|X_i-Y_i\|_2 .$$
The same metric is used for the two canonizations, and its quotient $\rho_\G$ is used for the quotient space.

\subsection{Experiment setup: ModelNet40 classification}
\begin{table}[ht]
\centering
\caption{Test accuracy of the Global MLP under varying training set sizes. The performance gap between Hilbert and lexicographical sorting widens as data becomes scarcer, highlighting the regularizing effect of canonization.}
\label{tab:data_scarcity}
\begin{tabular}{lcccc}
\toprule
\textbf{Training Samples} & \textbf{None} $\uparrow$ & \textbf{Lexicographical} $\uparrow$ & \textbf{Hilbert} $\uparrow$ & \textbf{DeepSets} $\uparrow$ \\
\midrule
9840 (full) & 0.433 & 0.700 & 0.748 & \textbf{0.759} \\
4920 & 0.408 & 0.641 & 0.714 & \textbf{0.729}  \\
2460 & 0.343 & 0.573 & 0.673 & \textbf{0.683}  \\
1230 & 0.307 & 0.504 & 0.601 & \textbf{0.644}  \\
\bottomrule
\end{tabular}
\end{table}
In the ModelNet40 classification tasks, we sampled from each input CAD model point cloud, which consists of $1024$ points, and applied an initial batched affine normalization to $[0,1]^3$. For the Global MLP models, the input points are processed using one of three strategies: no canonization (baseline), lexicographical sorting (i.e., sorting points sequentially by their $x$, then $y$, then $z$ coordinates), or Hilbert curve sorting (using a space size of $2^{12}$). We additionally evaluate DeepSets as a permutation-invariant model.

For the Global MLP variants, following the ordering step, the coordinates are projected into a higher-dimensional space using Random Fourier Features to capture high-frequency geometric details. The resulting feature vector is then passed through a bottlenecked MLP consisting of residual blocks with dimensions $[256,128,64]$, yielding the final classification logits. DeepSets applies the corresponding pointwise MLP independently to each point and aggregates the resulting logits by summation.

To ensure a rigorous and fair comparison, we selected model-specific hyperparameters for the uncanonized baseline (\texttt{ply}), lexicographical sorting (\texttt{lex}), Hilbert curve sorting (\texttt{hilbert}), and DeepSets. The specific hyperparameters used for the final evaluation are detailed in Table~\ref{tab:hyperparams}. The selected hyperparameters for each model on ModelNet40 are also used for the corresponding experiments on ModelNet10.
\begin{table}[h]
\centering
\caption{Optimal hyperparameters identified via grid search for each ordering strategy.}
\label{tab:hyperparams}
\begin{tabular}{lcccccc}
\toprule
\textbf{Method} & \textbf{Learning Rate} & \textbf{Dropout} & \textbf{Weight Decay} & \textbf{Fourier Scale} & \textbf{Label Smoothing} & \textbf{Bands} \\
\midrule
No canon. & $1.63 \times 10^{-3}$ & 0.3 & 0.05 & 1.0 & 0.1 & 2 \\
Lexsort  & $6.41 \times 10^{-4}$ & 0.1 & 0.05 & 0.1 & 0.2 & 3 \\
Hilbert  & $9.61 \times 10^{-4}$ & 0.3 & 0.01 & 0.1 & 0.2 & 3 \\
\bottomrule
\end{tabular}
\end{table}

\paragraph{Robustness Across Seeds}
To account for variance in weight initialization and optimization dynamics, we evaluated the final optimal configurations across 5 independent random seeds. The averaged test accuracy and standard deviations are reported in Table~\ref{tab:seed_robustness}. The results demonstrate that the Hilbert curve canonization not only achieves the highest average accuracy but also exhibits the lowest variance across different initializations.

\begin{table}[h]
\centering
\caption{Test accuracy across 5 independent random seeds using the optimal hyperparameters. Hilbert canonization yields both higher performance and greater stability.}
\label{tab:seed_robustness}
\begin{tabular}{lcc}
\toprule
\textbf{Method} & \textbf{Average Accuracy} $\uparrow$ & \textbf{Standard Deviation} \\
\midrule
No canon. & 0.433 & 0.0249 \\
Lexsort  & 0.700 & 0.01358 \\
Hilbert  & \textbf{0.748} & 0.00944 \\
\bottomrule
\end{tabular}
\end{table}

\subsection{Experimental Setup: Modelnet40 Data Scarcity Experiment}
To simulate data scarcity, we subsample the ModelNet40 training set using strides of 1, 2, 4, and 8, resulting in training subsets of 9840, 4920, 2460, and 1230 samples, respectively. Results are shown in Table  \ref{tab:data_scarcity}.

\subsection{Experimental Setup: PCA group averaging}
\label{subsec:ablation_pca}
We use a DeepSets architecture that is fully invariant to point order permutations but strictly non-invariant to the $\{-1, 1\}^3$ sign flips. We intentionally employ a lightweight, low-capacity version of this architecture. A highly over-parameterized model might achieve higher absolute accuracies, potentially masking the generalization gaps we aim to observe; by constraining the network capacity, we ensure that any performance differences are strictly attributable to the chosen alignment strategy. Each input point cloud consists of $1024$ points from the ModelNet40 dataset. We compare four approaches: (1) \textbf{Pure PCA}, which leaves the sign ambiguity unresolved (acting as a baseline or potentially \emph{poor} canonization due to eigensolver discontinuities); (2) \textbf{Skewness}, a deterministic canonization mapping $c: K \to K$ that resolves the ambiguity by orienting axes based on the third moment; (3) \textbf{Frame Averaging}, an explicitly $\G$-invariant model constructed by averaging the predicted class probabilities across all $8$ elements of the group, implemented in log-space for numerical stability; and (4) \textbf{Random Frame}, which applies elements of $\G$ as random data augmentation during training. All models are optimized using Adam for 1600 epochs and evaluated across 5 different random seeds.

\begin{figure}[htbp]
    \centering
    \includegraphics[width=0.7\linewidth]{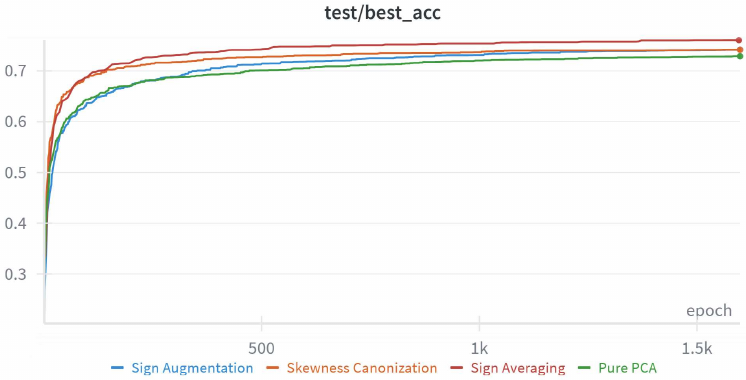}
    \caption{\textbf{Rotation Symmetry Accuracy.} The curves illustrate the comparative performance of the four approaches: Random Frame (blue), Frame Averaging (red), Skewness (orange), and the Pure PCA baseline (green). Frame Averaging achieves the strongest overall performance. While Random Frame and Skewness reach comparable final results—both significantly outperforming the Pure PCA baseline—the augmentation-based Random Frame method exhibits a noticeably slower convergence rate than deterministic canonization via Skewness.}
    \label{fig:loss_pca}
\end{figure}

\subsection{Experimental Setup: Rotated MNIST}
In this set of experiments, we compare the performance of several image rotation-invariant models. We considered four models: simple CNN, CN(p4 frozen)-CNN, CN(p4)-CNN, and AvgCNN. The backbone for all four models is the same and additional layers are added for CN(p4)-CNN and AvgCNN.
In order to have the same setting, including the data split we rerun all models, including CN(p4)-CNN.
We used a learning rate of $10^{-3}$, a batch size of $256$, and an AdamW optimizer. Each experiment is run for $5$ different seeds, and the mean and standard deviation are reported.

\subsection{Experimental Setup: Deep Weight Spaces}
\label{app:dws_setup}

We consider the INR classification setting on MNIST-INR and Fashion-MNIST-INR~\citep{pmlr-v202-navon23a}. In these datasets, each image is represented by the parameters of an implicit neural representation (INR), and the objective is to predict the image class directly from the INR weights. For both datasets, we use the original train/validation/test splits containing $55{,}000$, $5{,}000$, and $10{,}000$ INRs, respectively.

We compare three approaches: (1) \textbf{Simple MLP}, a non-invariant model operating directly on the flattened INR parameters; (2) \textbf{Canonized MLP}, where each INR is first mapped to a deterministic representative of its hidden-neuron permutation orbit and then provided to the MLP; and (3) \textbf{DWSNet}, a specialized permutation-invariant weight-space model~\citep{pmlr-v202-navon23a}. The Simple MLP and Canonized MLP use the same backbone, so that their comparison isolates the effect of canonization.

For both the Simple MLP and Canonized MLP, generic INR augmentation and hidden-neuron permutation augmentation are disabled. For the Canonized MLP, canonization is performed once as a preprocessing step on all INRs. Since canonization reorders the parameter coordinates, we recompute the normalization statistics on the canonized training set and use these statistics to normalize the canonized training, validation, and test examples. The original dataset statistics are used for the non-canonized models.

Both MLP variants consist of four hidden layers of width $32$ with batch normalization. We train them for $100$ epochs using AdamW with learning rate $5\times10^{-3}$, weight decay $5\times10^{-4}$, and batch size $512$. Each experiment is repeated for three seeds, $0,1,2$. For each seed, we select the checkpoint achieving the highest validation accuracy and report the test accuracy of that checkpoint. Mean accuracy and standard deviation across the three seeds are reported. We now describe the details of the deterministic weight-space canonization procedure.

\subsection{Deep Weight-Space Canonization}
\label{app:dws_canonization}

We canonize the hidden-neuron permutations sequentially from the input
towards the output. After the ordering of layer $\ell-1$ has been fixed,
each neuron in hidden layer $\ell$ is assigned a deterministic
lexicographic key constructed from its incoming weights, bias, and an
order-invariant summary of its outgoing weights. The hidden neurons are
then sorted according to this key, and the induced permutation is
propagated to the corresponding dimension of the next weight matrix.
Repeating this procedure layer by layer yields a single deterministic
representative of the permutation orbit for generic networks. The
canonization is performed once as a preprocessing step; the resulting
canonized parameters are subsequently provided to exactly the same MLP
architecture used for the non-canonized baseline.

\end{document}